\documentclass{article}

\usepackage{PRIMEarxiv}
\usepackage{xcolor}
\usepackage{multirow}
\usepackage[utf8]{inputenc} 
\usepackage[T1]{fontenc}    
\usepackage{hyperref}       
\usepackage{url}            
\usepackage{booktabs}       
\usepackage{amsfonts}       
\usepackage{nicefrac}       
\usepackage{microtype}      
\usepackage{lipsum}
\usepackage{fancyhdr} 
\usepackage{amsmath}
\usepackage{graphicx}       
\graphicspath{{media/}}     
\usepackage{algorithm}
\usepackage{algorithmic}

\usepackage{booktabs}
\usepackage{newfloat}
\usepackage{listings}
\floatstyle{ruled}
\newfloat{listing}{tb}{lst}{}
\floatname{listing}{Listing}

\usepackage{soul}
\usepackage[utf8]{inputenc}
\usepackage{amsmath}
\usepackage{amssymb}
\usepackage{amsthm}
\newtheorem{proposition}{Proposition}
\numberwithin{proposition}{section}
\usepackage[switch]{lineno}

\numberwithin{proposition}{section}
\newtheorem{corollary}{Corollary}
\usepackage[switch]{lineno}
\usepackage{xcolor}
\newtheorem{lemma}{Lemma}
\newtheorem{remark}{Remark}
\usepackage{mathtools}

\title{When Features Beat Noise: A Noise-Driven Hypothesis Testing Method for Feature Selection
}

\author{
  Mousam Sinha$^{*}$\\
  Indian Statistical Institute,\\ Kolkata\\
 \\
   \And
  Tirtha Sarathi Ghosh$^{*}$\\
  Jadavpur University,\\ Kolkata\\
  \\
\And
  Koushik Biswas$^{*}$\\
  IIIT Delhi,\\New Delhi\\
\\
  \And
  Ridam Pal$^{*}$\\
  IIIT Delhi,\\New Delhi\\
}

\begin{document}
\maketitle
\begingroup
\renewcommand{\thefootnote}{}
\footnotetext[0]{* These authors contributed equally to this work.}
\endgroup

\begin{abstract}
Feature selection has remained a daunting challenge in machine learning and artificial intelligence, where increasingly complex, high-dimensional datasets demand principled strategies for isolating the most informative predictors. Wrapper-based methods have seen broad adoption; many established techniques such as Boruta, Recursive Feature Elimination, and other wrapper-based statistical methods from notable bioinformatics studies incur substantial computational cost, while others offer no definite, statistically driven stopping criteria or assess the significance of their importance scores. This paper proposes a novel feature selection method that addresses these limitations. Our approach introduces multiple random noise features and evaluates each feature's importance against the maximum importance value among these noise features, incorporating a non-parametric bootstrap-based hypothesis testing framework to establish a solid theoretical foundation. We establish the conceptual soundness of our approach through statistical derivations that articulate the principles guiding the design of our algorithm. To assess reliability and practical utility, we validated the method on controlled simulations and diverse real-world datasets, where it consistently outperformed established feature-selection approaches including Boruta, Model-X Knockoffs, RFE, LASSO, Elastic Net, and mRMR in recovering meaningful signals. Hence, the method emerges as a robust algorithm for principled feature selection, enabling the distillation of informative predictors that support reliable inference, enhanced predictive performance, and efficient computation.
\end{abstract}


\section{Introduction}
With recent advancements in Machine learning, Deep learning, and Transformer architectures, various applications of high-dimensional data have grown rapidly, fundamentally transforming the landscape of classical predictive modeling \cite{borah2024review,rong2019feature}. Healthcare, bioinformatics, finance, marketing, and materials science now routinely generate structured datasets with thousands of features, capturing complex data patterns through diverse signals. Although rich with potential, the surge in dimensionality also poses a critical challenge. It becomes essential to retain truly informative variables for downstream predictive tasks. High-dimensional data often contain redundant, irrelevant, or noisy variables, which can mask the underlying signal. This phenomenon makes it difficult for the machine learning models to capture the true data patterns, leading to reduced predictive reliability. Isolating the relevant features would effectively enhance predictive performance, accelerate training, reduce overfitting, and improve the interpretability of complex models ~\cite{brahim2023bird}. Hence, feature selection is pivotal in developing robust and interpretable machine learning and deep learning models.  In deep learning, although representation learning inherently performs feature abstraction, implementing feature selection strategies can further lead to better model efficacy, faster convergence, reduced computational cost, and enhanced model interpretability \cite{figueroa2021towards}.

Traditional approaches for feature selection are broadly categorized as filter \cite{sanchez2007filter}, wrapper \cite{el2016review}, or embedded methods \cite{wang2015embedded}, each leveraging unique advantages \cite{jovic2015review}. Filter methods leverage statistical metrics to assess features independently, whereas wrapper methods evaluate subsets of features based on model performance. Embedded techniques incorporate feature selection during model training, often using regularization or tree-based strategies. Recent advancements such as SelectKBest, Recursive Feature Elimination (RFE) \cite{chen2007enhanced}, Lasso \cite{fonti2017feature}, and tree-based feature importance have shown higher efficacy in identifying relevant features. Among these, Boruta \cite{kursa2010boruta}, a wrapper method built on random forests, stands out for its robustness in selecting all relevant features. The Extra-Trees classifier \cite{alfian2022predicting}has gained prominence for its ensemble-based importance estimation and resistance to overfitting.

Despite the proliferation of feature selection techniques in modern machine learning pipelines, several persistent limitations hinder their efficacy. Firstly, most commonly used algorithms such as Recursive Feature Elimination (RFE) lack an explicit statistical criterion to determine the optimal stopping point for feature retention. This absence of a quantitative threshold often leaves practitioners reliant on ad hoc or empirical decisions, which may compromise model performance. Secondly, when these selection procedures are coupled with computationally intensive models like tree-based models (e.g., Random Forest, XGBoost), the exhaustive need for hyperparameter tuning at each iteration makes the process expensive, particularly on large-scale or high-dimensional datasets.
While methods like the randomized shadow method address this by comparing features to their randomized ``shadow'' versions, this comes at the cost of doubling the feature space, exacerbating memory usage. Moreover, many existing approaches are fundamentally heuristic, empirical, and data-driven, lacking a principled mechanism to guarantee optimality in information retention or generalization. Perhaps, even after selecting a subset of features, current methods do not ensure that these features do not contribute to model overfitting. These challenges highlight a pressing need for feature selection strategies that are both methodically robust and computationally efficient for distinguishing informative signals from noise while safeguarding against model overfitting.

\section{Related Work and Our Solution}

In this paper, we aim to fill this gap by proposing a robust and interpretable feature selection technique that leverages model bootstrapping, noise calibration, and rigorous statistical testing to yield reliable predictive features with high generalization potential.

In Section~\ref{sec:Methods}, we introduce our proposed \textit{Noise-Augmented Bootstrap Feature Selection (NABFS)} framework, which integrates controlled noise augmentation, resampling, and non-parametric statistical inference to robustly identify informative features. The method begins by augmenting the real dataset with several artificially generated ``noise'' features drawn from a fixed random distribution, which act as a benchmark for detecting spurious importance. Using bootstrap resampling, the predictive model is repeatedly fitted to multiple resampled datasets, and feature importance is computed for both real and synthetic variables in each iteration. For every bootstrap replicate, the relative importance of each real feature is quantified against the maximum importance observed among the noise features.  To assess whether a given feature consistently outperforms the noise benchmark, a one-sided Wilcoxon signed-rank test is performed across bootstrap iterations, yielding a p-value that captures the statistical evidence of informative intelligence. To control the overall family-wise error rate, p-values are adjusted using the Holm–Bonferroni method, a rigorous multiple testing correction. Finally, features with adjusted p-values less than the predefined significance threshold, \textit{~$\alpha$}, are retained as significantly informative features. This approach not only enhances robustness against benchmarked relations and overfitting but also provides a statistically grounded mechanism for distinguishing meaningful signals from noise within complex and high-dimensional data.

In Sections~\ref{sec:experiments}, we benchmark the proposed \textit{Noise-Augmented Bootstrap Feature Selection (NABFS)} method against established approaches using both controlled simulation studies and real-world datasets. This dual evaluation framework enables a comprehensive assessment of the method’s effectiveness across diverse data settings, thereby demonstrating its robustness and generalizability.

In Section~\ref{sec:Results}, we present a detailed discussion of our findings, highlighting the comparative performance of the proposed method against established feature selection techniques and outlining the future scope of this study. The results demonstrate that our approach consistently outperforms state-of-the-art methods across multiple evaluation settings, establishing a strong empirical benchmark. Comprehensive assessments across diverse machine learning models (logistic regression, random forest, gradient boosting, and extreme gradient boosting classifiers) reveal that models trained on the selected feature subsets achieve superior performance compared to their full-feature counterparts, depicting the robustness of the proposed method in both linear and non-linear regimes. These findings reaffirm that even in an era dominated by complex architectures and data abundance, the principled selection of informative features remains indispensable for enhancing predictive accuracy, interpretability, and computational efficiency in artificial intelligence systems.

Hence, this paper provides a comprehensive theoretical foundation for the proposed method, supported by rigorous empirical validation. An overview of the complete NABFS pipeline is presented in Figure~\ref{fig2}. The effectiveness of the proposed algorithm is subsequently demonstrated through extensive evaluations on diverse datasets, where it is benchmarked against state-of-the-art feature selection techniques to establish its robustness and practical utility.

\begin{figure*}[t]
\centering
\includegraphics[width=0.7\textwidth]{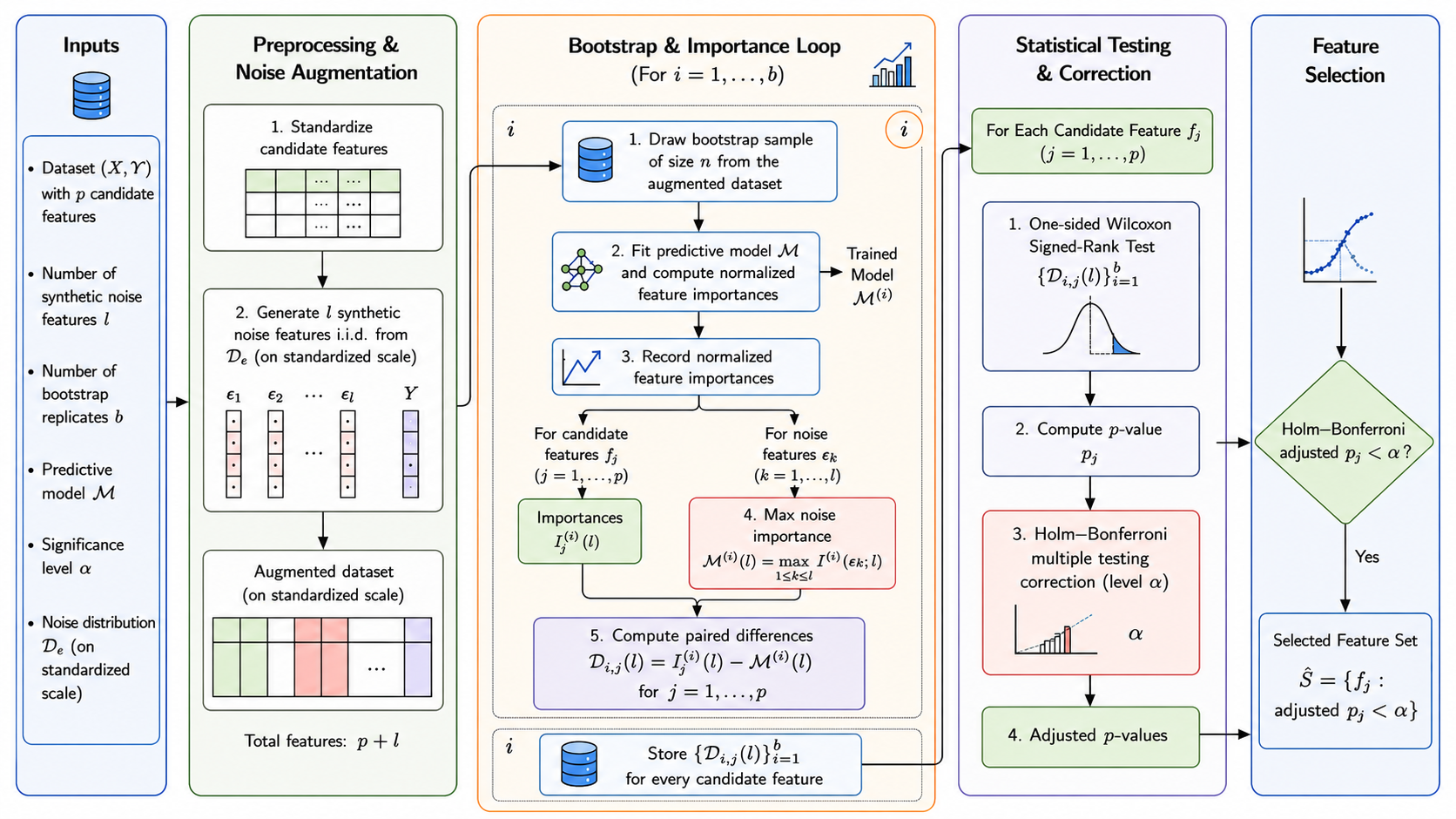} 
\caption{Overview of the proposed \textit{Noise Augmented Bootstrap Feature Selection (NABFS)} algorithm. Candidate features are standardized and augmented with synthetic noise features, evaluated through bootstrap-based feature importance estimation, statistically compared against the maximum noise importance, and selected using Wilcoxon signed-rank testing with Holm–Bonferroni correction.}
\label{fig2}
\end{figure*}

\section{Method}\label{sec:Methods}

\subsection{Population Assumptions}

We consider a population characterized by a response variable $Y$ and $p$ candidate features 
\[
f_1, f_2, \dots, f_p.
\]
Each feature may or may not be statistically related to the response.

Our primary goal is to identify which features exhibit any form of statistical dependence with $Y$. 
Formally, for each feature $f_j$, we test the population-level hypothesis:
\[
H_{0j}: Y \;\perp\!\!\!\perp\; f_j
\quad \text{versus} \quad
H_{1j}: Y \;\not\!\perp\!\!\!\perp\; f_j,
\]
where $\perp\!\!\!\perp$ denotes statistical independence. 
Under the null, the feature $f_j$ carries no information about $Y$.

To quantify this relationship, we define a \emph{population-level importance measure} $\mathcal{I}(f_j, Y)$ that measures the dependence between $f_j$ and $ Y$ at the population level, such as mutual information, correlation, regression coefficient magnitude, or feature importance metrics in tree-based algorithms (e.g., XGB) like Gain, SHAP, etc.

Equivalently, the hypotheses can be expressed as:
\[
H_{0j}: \mathcal{I}(f_j, Y) = 0,
\qquad
H_{1j}: \mathcal{I}(f_j, Y) > 0.
\]
The measure $\mathcal{I}(f_j, Y)$ is assumed to be nonnegative, for instance when defined as an absolute correlation, mutual information, or coefficient magnitude. 
Hence, the alternative hypothesis corresponds to the presence of strictly positive population importance, indicating statistical dependence between $f_j$ and $Y$.

\subsection{Approximate solution}

Because population importances $\mathcal{I}(f)$ are not directly observable, 
we construct an approximate solution based on resampling and empirical feature 
importance scores. The idea is to solve this hypothesis testing problem by repeatedly perturbing the data, estimating feature importances, 
and comparing candidate features against synthetic noise variables.

The procedure begins by augmenting the given data set with $l$ artificial features
drawn independently from a fixed distribution $\mathcal{D}_\varepsilon$. 
These noise variables play the role of a baseline: any real feature that cannot 
consistently outperform them should not be considered informative.

To approximate the distribution of importances, we employ the bootstrap. 
Specifically, we generate $b$ bootstrap replicates of the augmented dataset, 
each of size $n$. For each replicate, we fit a predictive model, treated as a 
hyperparameter of the framework, and extract feature importance scores. 
Denote by $I^{(i)}_j(l)$ the importance of candidate feature $f_j$ in the 
$i$-th bootstrap when $l$ noise features are present, and by
\[
M^{(i)}(l) \;=\; \max_{1 \leq k \leq l} I^{(i)}(\varepsilon_k; l)
\]
the maximum importance among the $l$ noise variables in the same replicate.

This construction yields a paired sequence of differences
\[
D_i^{(j)}(l) \;=\; I^{(i)}_j(l) - M^{(i)}(l), \qquad i=1,\dots,b,
\]
one per bootstrap replicate. For notational simplicity, we suppress the feature index $j$ in $D_i^{(j)}(l)$ in what follows. 
We then test the one-sided alternative that the candidate feature tends to exceed the noise maximum using the 
\emph{Wilcoxon signed-rank (WSR)} statistic
\[
T^+(l) \;=\; \sum_{i=1}^b R_i(l)\,\mathbf{1}\{D_i(l)>0\},
\]
where $R_i(l)$ is the rank of $|D_i(l)|$ among the nonzero 
$\{|D_1(l)|,\dots,|D_b(l)|\}$ (mid-ranks for ties; zero differences dropped).
Large values of $T^+(l)$ indicate suppression in favor of the alternative $H_{1j}$.
A one-sided $p$-value is obtained from the exact or large-sample WSR null (suppressing the feature index $j$).

Finally, because multiple features are tested simultaneously, $p$-values are 
adjusted using the Holm--Bonferroni procedure to control the family-wise error 
rate (FWER) at level $\alpha$. Features that remain significant after adjustment 
are retained.

\subsection{Algorithmic Summary}


Algorithm~\ref{alg:noise_fs} outlines the proposed 
Noise-Augmented Bootstrap Feature Selection (NABFS) method.









\begin{algorithm}[t]
\caption{Noise-Augmented Bootstrap Feature Selection (NABFS)}
\label{alg:noise_fs}
\begin{algorithmic}[1]

\REQUIRE Dataset $(X,Y)$ with $p$ candidate features; number of synthetic noise features $l$; number of bootstrap replicates $b$; predictive model $\mathcal{M}$; significance level $\alpha$; noise distribution $\mathcal{D}_{\varepsilon}$.

\ENSURE Selected feature set $\widehat{\mathcal{S}} \subseteq \{f_1,\dots,f_p\}$.

\STATE Standardize the candidate features and augment the dataset with $l$ synthetic noise features sampled i.i.d. from $\mathcal{D}_{\varepsilon}$ on the standardized scale, so that the noise importances are exchangeable with a null candidate importance.

\FOR{$i = 1,\ldots,b$}
    \STATE Draw a bootstrap sample of size $n$ from the augmented dataset.
    \STATE Fit the predictive model $\mathcal{M}$ and compute normalized feature importances.
    \STATE Record the importance $I^{(i)}_j(l)$ for each candidate feature $f_j$.
    \STATE Compute the maximum noise importance
    \[
    M^{(i)}(l)=\max_{1\le k\le l} I^{(i)}(\varepsilon_k;l).
    \]
    \FOR{$j = 1,\ldots,p$}
        \STATE Compute the paired difference
        \[
        D_{i,j}(l)=I^{(i)}_j(l)-M^{(i)}(l).
        \]
    \ENDFOR
\ENDFOR

\FOR{$j = 1,\ldots,p$}
    \STATE Compute the one-sided Wilcoxon signed-rank statistic $T^{+}_j(l)$ using
    $\{D_{i,j}(l)\}_{i=1}^{b}$.
    \STATE Obtain the corresponding p-value $p_j$.
\ENDFOR

\STATE Apply the Holm--Bonferroni procedure to $\{p_1,\dots,p_p\}$ at significance level $\alpha$.

\STATE \textbf{Return} $\widehat{\mathcal{S}}
=\{f_j:\text{adjusted p-value}<\alpha\}$.

\end{algorithmic}
\end{algorithm}

\subsection{Practical Considerations}

Several elements of the framework are treated as hyperparameters:

\begin{itemize}
    \item \textbf{Number of noise features ($l$):} Larger $l$ yields a stricter benchmark, 
    reducing false positives but lowering power. Our theoretical analysis 
    (Section~\ref{sec:theory}) shows that both Type~I error and power decrease monotonically with $l$ under mild assumptions.

    \item \textbf{Number of bootstrap replicates ($b$):} Increasing $b$ stabilizes the empirical 
    distributions but raises computational cost. In practice, $b$ is chosen to balance stability 
    with computational feasibility.

    \item \textbf{Choice of predictive model ($\mathcal{M}$):} The framework is model-agnostic: 
    any learner that provides feature importance scores may be specified. Examples include 
    tree-based models, linear models with coefficients as importance measures, or neural 
    networks with attribution scores. In practice, one may select $\mathcal{M}$ from a 
    candidate set by held-out predictive performance or stability of selections.

    \item \textbf{Significance level ($\alpha$):} The target FWER level is a fixed input that 
    determines the stringency of the Holm– Bonferroni correction. Smaller values of $\alpha$ 
    reduce false discoveries but may also exclude weak signals.

    \item \textbf{Noise distribution ($\mathcal{D}_\varepsilon$):} Noise features are sampled 
    independently from a fixed distribution. While Gaussian noise is common, other distributions 
    (e.g., uniform) may be used. The choice of $\mathcal{D}_\varepsilon$ affects the scale of 
    the benchmark but not the validity of the testing framework.
\end{itemize}

The procedure is approximate, relying on bootstrap distributions of empirical importances 
as proxies for unobservable population importances. Despite this approximation, as demonstrated 
in Section~\ref{sec:experiments}, the method provides a robust and practical tool for 
distinguishing informative features from noise. In applications, $\alpha$ is typically fixed at a conventional level (e.g., $0.05$), 
and $\mathcal{D}_\varepsilon$ can be taken as standard Gaussian. The key tunable parameters are 
the number of noise features $l$ and the underlying model $\mathcal{M}$; their choices 
may be guided by cross-validation, predictive performance on held-out data, and stability analyses 
of selected feature sets (see Section~\ref{sec:experiments}).

\section{Theoretical Properties}
\label{sec:theory}
The procedure described in Section~3.2 provides an approximate solution to the hypothesis testing formulation introduced in Section~3.1. In this section, we establish its theoretical properties. 
We show that the number of synthetic noise features $l$ 
governs a key trade-off between Type~I error and statistical power. 
Proofs for all statements in this section have been provided in the \textit{Technical Appendix}.

\subsection{Type I Error and Power as Functions of \texorpdfstring{$l$}{l}}

We next study how the number of synthetic noise features, $l$, affects the statistical 
properties of the test. Recall that for each bootstrap replicate $i$, the paired difference is
\[
D_i(l) \;=\; I^{(i)}_j(l) - M^{(i)}(l),
\qquad
M^{(i)}(l) \;=\; \max_{1\le k\le l} I^{(i)}(\varepsilon_k; l),
\]
and the Wilcoxon signed-rank statistic is
\[
T^+(l) \;=\; \sum_{i=1}^b R_i(l)\,\mathbf{1}\{D_i(l)>0\},
\]
where $R_i(l)$ is the rank of $|D_i(l)|$ among the nonzero absolute values.

\begin{proposition}[Type I Error is Decreasing in $l$]
For any fixed feature under the null hypothesis, the Type~I error probability decreases 
as $l$ increases. If
\[
\alpha_j(l) \;=\; \mathbb{P}\!\left( T^+(l)\ge c_\alpha(b) \;\middle|\; H_{0j} \text{ true} \right),
\]
then
\[
\alpha_j(l+1) \;\le\; \alpha_j(l).
\]
\end{proposition}

\noindent
\textit{Intuition:} Adding more noise features raises the replicate-wise maximum noise benchmark 
$M^{(i)}(l)$, which weakly lowers the paired differences $D_i(l)$; hence the WSR statistic $T^+(l)$ 
cannot increase, making false rejections less likely.

\begin{proposition}[Power is Decreasing in $l$]
For any fixed feature under the alternative hypothesis, the power of the test decreases 
as $l$ increases. If
\[
\beta_j(l) \;=\; \mathbb{P}\!\left( T^+(l)\ge c_\alpha(b) \;\middle|\; H_{1j} \text{ true} \right),
\]
then
\[
\beta_j(l+1) \;\le\; \beta_j(l).
\]
\end{proposition}

\noindent
\textit{Intuition:}Increasing $l$ makes the noise maximum stricter, so even truly 
informative features are less likely to consistently exceed the benchmark. 
Because $D_i(l)$ is non-increasing in $l$ by Assumption~(A1), the WSR statistic 
$T^+(l)$ is also pointwise non-increasing (see \textit{Technical Appendix: WSR monotonicity lemmas}), 
and thus the probability of crossing any fixed rejection threshold decreases.

\subsection{Exact Null Behavior and Finite-Sample Guarantees}

The propositions above describe how the test behaves as $l$ varies. We now characterize the null behavior exactly for a fixed $l$. Under $H_{0j}$, the candidate feature carries no information about $Y$, so within each replicate its importance is statistically interchangeable with the noise importances. We formalize this as null exchangeability: within replicate $i$, the values $I^{(i)}_j(l)$ and $I^{(i)}(\varepsilon_1;l),\dots,I^{(i)}(\varepsilon_l;l)$ are exchangeable. This holds, in particular, when the benchmark is built from permuted copies of the candidate feature, which share its marginal distribution while being independent of $Y$ under the null. Under this condition and continuity (A2), the following hold; proofs are provided in the \textit{Technical Appendix}.

\begin{proposition}[Exact null exceedance rate]
Under $H_{0j}$, null exchangeability, and (A2),
\[
\mathbb{P}\!\left( D_i(l)>0 \;\middle|\; H_{0j} \right) \;=\; \frac{1}{l+1},
\qquad i=1,\dots,b.
\]
The identity is exact and does not require the replicates to be independent.
\end{proposition}

\begin{proposition}[Null mean of the positive count]
Let $S(l)=\sum_{i=1}^b \mathbf{1}\{D_i(l)>0\}$ count the replicates in which the candidate exceeds the noise maximum. Then
\[
\mathbb{E}\!\left[ S(l) \;\middle|\; H_{0j} \right] \;=\; \frac{b}{l+1},
\]
irrespective of any dependence across bootstrap replicates.
\end{proposition}

\begin{proposition}[Finite-sample Type~I control]
Consider the one-sided WSR test calibrated under the symmetric signed-rank null. Under $H_{0j}$, null exchangeability, (A2), and (A3): for $l=1$ the signed differences are symmetric about zero, and the test attains exact level $\alpha$; for $l\ge 2$ the per-replicate exceedance rate $1/(l+1)<\tfrac12$ makes $T^+(l)$ stochastically smaller, so the achieved Type~I error satisfies $\alpha_j(l)\le\alpha$.
\end{proposition}

\noindent\textit{Interpretation}.The rate $1/(l+1)$ pins the expected number of false wins regardless of bootstrap-induced dependence, giving a finite-sample Type~I bound. The same exact-rate argument yields a win-probability characterization under the alternative, a consistency result, an exactly calibrated sign test based on $S(l)\sim\mathrm{Binomial}(b,1/(l+1))$ (usable with Benjamini--Hochberg for FDR control), a principled rule for choosing $l$, and a relaxation of (A1) to an on-average condition~(A1$'$); all statements and proofs are in the \textit{Technical Appendix}.

\subsection{Practical Implications}

The two propositions show that the noise parameter $l$ governs a fundamental trade-off: 
larger $l$ reduces both Type~I error and statistical power, producing a more conservative 
procedure; smaller $l$ yields higher power but risks including spurious features. 
In practice, $l$ should be treated as a hyperparameter and tuned to balance 
conservativeness with sensitivity to weak signals.

\section{Experimental Setup}\label{sec:experiments}
We evaluated NABFS through simulations and real-world datasets. Simulated data were generated under both $n>p$ and $p>n$ settings using a logistic model with compound-symmetric correlations. For the $n>p$ setting, $n \in \{500,1000,3000\}$, $p=50$, $k=20$, the correlation $\rho$ was varied across multiple levels, and $l \in \{1,\ldots,5\}$ synthetic noise features were added. Each configuration was repeated 20 times. An analogous experimental setup was used for the $p>n$ setting. Results were compared with Boruta \cite{kursa2010boruta} and Model-X Knockoffs \cite{candes2018panning} using \textit{Power} \cite{cohen2013statistical}, \textit{Type~I Error} \cite{banerjee2009hypothesis}, and \textit{Jaccard Index} \cite{real1996probabilistic}.

Real-world evaluation included clinical (UCI MI, Parkinson’s) \cite{golovenkin2020trajectories,sakar2013collection}, ICU shock prediction (ShockModes) \cite{pal2022shockmodes}, CRISPR/Cas9 MD datasets \cite{panda2024deciphering}, financial risk (fraud detection, loan default)\cite{awoyemi2017credit,robinson2024loan}, and various other miscellaneous datasets \cite{silva2008using,jethwa2024comprehensive}. Evaluation metrics for real datasets included the minority-class \textit{F1}, \textit{MCC}, and \textit{PR-AUC} for classification, and \textit{RMSE} and \textit{$R^2$} for regression. Full details of the experimental setup, hyperparameters, computational (Time \& Space) complexity, and computing infrastructure have been provided in the \textit{Technical Appendix}.

\section{Results}
\label{sec:Results}

The efficacy of our algorithm for simulation studies has been demonstrated in \textit{Technical Appendix}. Building on the extensive simulation studies that demonstrated the effectiveness of the proposed Noise-Augmented Bootstrap Feature Selection (NABFS) method, we further evaluated its practical utility across a diverse collection of real-world datasets (healthcare, finance, and biological sequences). This empirical assessment was designed to test whether the advantages observed in controlled synthetic settings are also translated to heterogeneous, noisy, and domain-specific datasets. NABFS was systematically benchmarked against established feature selection techniques, including Boruta, Recursive Feature Elimination (RFE), LASSO, Elastic Net, mRMR, and Model-X Knockoffs. Because several datasets are class-imbalanced, we report the minority-class F1, the Matthews correlation coefficient (MCC), and the area under the precision-recall curve (PR-AUC) as the primary classification metrics, together with the fraction of retained features. Hyperparameter selection for our feature selection technique was based upon AUC-based performance trends obtained across multiple noise intensities for both the model base learners (Logistic regression and XGBoost). For each base learner (Logistic and XGBoost), AUC values were evaluated across multiple noise intensities, and the final combination model–noise configuration was chosen based on the highest observed AUC, enabling principled, data-driven selection rather than heuristic tuning. The logistic regression–based model used a significance threshold of $\alpha = 0.01$, whereas the XGBoost-based model adopted $\alpha = 0.05$, reflecting differing sensitivity to linear and non-linear signal patterns.

\begin{table*}[htbp]
\centering
\small
\renewcommand{\arraystretch}{1.2}
\resizebox{\textwidth}{!}{%
\begin{tabular}{lcccccccccccc}
\hline
 & \multicolumn{4}{c}{\textbf{Our Technique (NABFS)}}
 & \multicolumn{4}{c}{\textbf{Boruta}}
 & \multicolumn{4}{c}{\textbf{Model-X Knockoffs}} \\
\cline{2-13}
\textbf{Dataset}
& \textbf{\% Ret.} & \textbf{Min-F1} & \textbf{MCC} & \textbf{PR-AUC}
& \textbf{\% Ret.} & \textbf{Min-F1} & \textbf{MCC} & \textbf{PR-AUC}
& \textbf{\% Ret.} & \textbf{Min-F1} & \textbf{MCC} & \textbf{PR-AUC} \\
\hline
Airline Customer Satisfaction
& 100.0 & 0.834 & 0.624 & 0.926
& 100.0 & 0.834 & 0.624 & 0.926
& 95.2 & 0.834 & 0.624 & 0.926 \\
Credit Card Fraud
& 24.1 & 0.619 & 0.639 & 0.668
& 93.1 & 0.585 & 0.611 & 0.665
& 100.0 & 0.585 & 0.611 & 0.665 \\
Parkinson's Disease Speech
& 0.3 & 0.847 & 0.154 & 0.885
& 5.8 & 0.865 & 0.294 & 0.936
& 100.0 & 0.861 & 0.251 & 0.937 \\
Septic Shock (MIMIC)
& 6.7 & 0.240 & 0.065 & 0.235
& 3.5 & 0.190 & 0.043 & 0.244
& 100.0 & 0.190 & 0.043 & 0.244 \\
\hline
\end{tabular}
}
\caption{Performance comparison of NABFS against Boruta and Model-X Knockoffs on binary classification datasets, using minority-class F1, MCC, and PR-AUC together with the percentage of retained features. NABFS attains competitive predictive performance while retaining substantially fewer features. The full comparison against RFE, LASSO, Elastic Net, and mRMR, and the multiclass results, are reported in the \textit{Technical Appendix}.}
\label{tab:feature_results}
\end{table*}

Across the binary classification datasets, NABFS retained substantially fewer features while remaining competitive with the baselines (Table~\ref{tab:feature_results}). On Credit Card Fraud, NABFS used only 24.1\% of the features yet attained the highest minority-class F1 (0.619), MCC (0.639), and PR-AUC (0.668) among the compared methods, whereas Boruta and Model-X Knockoffs retained 93--100\% of the features for comparable or lower scores. On Airline Customer Satisfaction, NABFS matched the baselines (minority-F1 0.834, MCC 0.624, PR-AUC 0.926) at full retention. On the highly imbalanced Septic Shock cohort, NABFS improved minority-F1 (0.240 versus 0.190) and MCC (0.065 versus 0.043) over the baselines while using 6.7\% of the features. On Parkinson's Disease Speech, NABFS produced an extremely compact subset (two features, 0.3\%) at some cost in MCC, illustrating the stringency--power trade-off characterized in Section~\ref{sec:theory}. The multiclass CRISPR off-target results and the full comparison against RFE, LASSO, Elastic Net, and mRMR are reported in the \textit{Technical Appendix}. For regression (Table~\ref{tab:regression_results}), NABFS again favored compact subsets with competitive accuracy: it achieved the best $R^2$ on Food Delivery Times (0.648 using 28.6\% of the features) and remained within $0.01$ in $R^2$ of the strongest baseline on Graduate Admission, Social Anxiety, and Student Performance, while typically retaining far fewer features than Model-X Knockoffs.

\begin{table*}[htbp]
\centering
\small
\renewcommand{\arraystretch}{1.2}
\setlength{\tabcolsep}{12pt}
\resizebox{\textwidth}{!}{%
\begin{tabular}{lccccccccc}
\hline
\textbf{Dataset} &
\multicolumn{3}{c}{\textbf{Our Technique (NABFS)}} &
\multicolumn{3}{c}{\textbf{Boruta}} &
\multicolumn{3}{c}{\textbf{Model-X Knockoffs}} \\
\cline{2-10}
 & \textbf{\% Ret.} & \textbf{RMSE} & $\mathbf{R^2}$
 & \textbf{\% Ret.} & \textbf{RMSE} & $\mathbf{R^2}$
 & \textbf{\% Ret.} & \textbf{RMSE} & $\mathbf{R^2}$ \\
\hline
Food Delivery Times & 28.6 & 12.200 & 0.648 & 14.3 & 12.404 & 0.636 & 100.0 & 12.362 & 0.638 \\
Graduate Admission  & 85.7 & 0.072 & 0.773 & 42.9 & 0.073 & 0.767 & 100.0 & 0.071 & 0.781 \\
Social Anxiety      & 23.3 & 1.101 & 0.737 & 30.0 & 1.099 & 0.738 & 56.7 & 1.099 & 0.738 \\
Student Performance & 20.5 & 3.534 & 0.158 & 12.8 & 3.524 & 0.163 & 100.0 & 3.540 & 0.155 \\
\hline
\end{tabular}}
\caption{Performance comparison of NABFS against Boruta and Model-X Knockoffs on regression datasets, using RMSE and $R^2$ together with the percentage of retained features. The full comparison against RFE, LASSO, Elastic Net, and mRMR is reported in the \textit{Technical Appendix}.}
\label{tab:regression_results}
\end{table*}

Together, these results demonstrate that the advantages of NABFS are not confined to controlled simulation regimes. Rather, the method generalizes effectively to real-world datasets. This combination of statistical rigor and practical efficiency positions NABFS as a robust and domain-agnostic framework for modern feature selection challenges. Detailed explanations of the results are provided in \textit{Technical Appendix}.

\section{Discussion and Conclusion} \label{sec:Discussion}
This study introduces Noise-Augmented Bootstrap Feature Selection (NABFS), a statistically principled approach for feature selection by incorporating formal hypothesis testing into the selection process. In contrast to heuristic probe-based methods that rely on a single noise comparison or empirically tuned thresholds, our method formulates feature relevance as a formal hypothesis-testing problem, leveraging bootstrap-stabilized importance estimates and a non-parametric Wilcoxon signed-rank test with Holm–Bonferroni correction. Such a foundation distinguishes NABFS from heuristic feature selection methods, providing greater confidence that the selected variables are not false discoveries but genuine predictors supported by evidence. 

In both controlled simulations and diverse real-world datasets, NABFS demonstrated strong performance relative to established feature selection techniques. In simulation studies, NABFS consistently achieved higher statistical power in detecting true relevant features compared to alternatives like Boruta and Model-X knockoffs, while maintaining a controlled false positive rate. Notably, even as inter-feature correlation increased, NABFS balanced sensitivity and specificity better than the alternatives, which tended to either suffer inflated false discoveries (in the case of Boruta) or remain overly conservative (in the case of knockoffs) with negligible gains in power. The selected feature sets from NABFS also showed greater stability and overlap with the true feature set than those of baseline methods, depicting its reliability in identifying meaningful signals without being misled by spurious correlations. This robust performance in simulations established a strong empirical benchmark for our technique. The efficacy of our method was translated into real-world applications, where NABFS delivered competitive or superior results using substantially reduced and more informative feature subsets. Our method consistently outperformed or achieved comparable predictive performance in comparison to established feature selection techniques (including Boruta, recursive feature elimination, and Model-X knockoffs). Together, these results asserted that the benefits of NABFS are not just confined to simulations; rather, the method generalizes effectively to real-world, multimodal, and high-dimensional datasets, achieving competitive predictive performance while yielding more compact and interpretable feature sets. In fact, models trained on the NABFS-selected features often achieved higher accuracy (AUC) than models using the full original feature set. 

A key strength of NABFS lies in its ability to explicitly manage the trade-off between statistical power and Type-I error. By varying the number of introduced noise features $l$, practitioners can systematically control the statistical stringency of the selection criteria. Increasing the size of the noise feature pool imposes a more stringent empirical null threshold that true signal features must surpass, thereby reducing susceptibility to false discoveries at the expense of statistical power. Conversely, using fewer noise probes makes the test more permissive, increasing sensitivity to weak signals but potentially allowing more false positives. In essence, the method provides a tunable parameter $l$ to attain a balance between Power and Type I errors. Empirically, we observed that power is a decreasing function of $l$, meaning one can calibrate $l$ to determine an optimal trade-off that suits the problem’s tolerance for false discoveries. This capability to regulate the error trade-off is a distinctive advantage, as it offers a principled means to maintain high discovery power while keeping false discoveries in check. NABFS is also notable for its computational efficiency and model-agnostic flexibility. Despite the added overhead of bootstrap resampling and noise augmentation, the method remains tractable even in high-dimensional settings.  Our observation states that NABFS often requires fewer model training iterations than exhaustive approaches like RFE (which retrains models for many subset sizes), and it also does not dramatically inflate memory usage like methods that double the feature space (e.g., naive implementations of Boruta, Knockoff). Furthermore, unlike certain feature selection techniques tied to specific model families (for instance, Boruta’s dependence on random forests), NABFS can be paired with any predictive modeling algorithm without modification. The model-agnostic nature of NABFS was demonstrated by its seamless integration with both linear base learners (logistic regression) and non-linear ensemble base learners (including XGBoost and Random Forest), relying solely on each learner’s feature importance mechanism without requiring algorithm-specific modifications. In summary, the proposed method emerges as a robust yet efficient method for principled feature selection, enabling the distillation of informative predictors that support reliable inference, improved model performance, and reduced computational cost.

Despite its empirical strengths, NABFS has some limitations that warrant careful consideration. First, the method relies on a base learner to estimate feature importance and is therefore sensitive to the stability and inductive biases of the underlying model. Second, the proposed procedure constitutes an approximate solution to the formal problem outlined in Section~\ref{sec:Methods}. While we attempted to control the family-wise error rate(FWER) through multiple hypothesis testing (via Holm–Bonferroni-type corrections), deriving strict theoretical guarantees of FWER control under our construction remains a future scope. The method further operates under a set of modelling assumptions described in \textit{Technical Appendix} Section. If these assumptions are violated, there is no formal guarantee that statistical power or Type~I error will behave as a monotonic function of the number of injected noise features (\(l\)). However, we provided intuition and empirical evidence suggesting that these assumptions are not overly restrictive in practical settings. Moreover, the current implementation injects synthetic noise drawn from a fixed Gaussian distribution, \( \mathcal{N}(0, 1) \), which constitutes an additional tunable hyperparameter. Although fixed in this study for reproducibility, practitioners may select alternative noise distributions and variances more appropriate to their domain and requirements of the data. Like other importance-based selection frameworks, NABFS can be challenged by highly collinear or redundant predictors. In such settings, feature importance can be diluted across correlated variables, leading the method to preferentially select dominant representatives while discarding redundant yet meaningful features. Similarly, excessively stringent selection thresholds arising either from overly conservative significance levels or large values of \(l\) can suppress weak but genuinely informative signals by elevating the synthetic-noise baseline. While we provide principled default settings, all hyperparameters (bootstrap iterations, noise magnitude, \(l\), and significance thresholds) remain user-tunable and should be optimized for domain-specific deployments. Finally, the current method is designed primarily to control the family-wise error rate. Extending the method to support explicit false discovery rate (FDR) control represents an important direction for future work. Collectively, these limitations highlight that while NABFS offers a practical, robust approach, its theoretical guarantees remain approximate, which in turn requires careful tuning and domain expertise for optimal deployment.

As future scope, there are several promising directions for extending our research. First, while we use Holm–Bonferroni in this paper, it would be valuable to develop stronger theoretical guarantees for FWER control under the effect of bootstrapping. This requires new theory, since bootstrap samples are not fully independent, and exploring this formally would make the method more rigorous. Another direction can be to improve the efficacy of NABFS when features are highly correlated. Possible extensions include grouping correlated variables, using conditional importance scores, or adding a special step to avoid losing meaningful but redundant features. Future versions of NABFS could also use adaptive noise, where the scale or distribution of the injected noise automatically adjusts to the dataset instead of being fixed. This would make the method easier to apply across diverse domains of datasets with heterogeneous statistical characteristics.

\bibliographystyle{unsrt}  
\bibliography{references}

\begin{thebibliography}{10}

\bibitem{borah2024review}
Kasmika Borah, Himanish~Shekhar Das, Soumita Seth, Koushik Mallick, Zubair Rahaman, and Saurav Mallik.
\newblock A review on advancements in feature selection and feature extraction for high-dimensional ngs data analysis.
\newblock {\em Functional \& Integrative Genomics}, 24(5):139, 2024.

\bibitem{rong2019feature}
Miao Rong, Dunwei Gong, and Xiaozhi Gao.
\newblock Feature selection and its use in big data: challenges, methods, and trends.
\newblock {\em Ieee Access}, 7:19709--19725, 2019.

\bibitem{brahim2023bird}
Samir Brahim~Belhaouari, Mohammed~Bilal Shakeel, Aiman Erbad, Zarina Oflaz, and Khelil Kassoul.
\newblock Bird’s eye view feature selection for high-dimensional data.
\newblock {\em Scientific Reports}, 13(1):13303, 2023.

\bibitem{figueroa2021towards}
Joaqu{\'\i}n Figueroa~Barraza, Enrique L{\'o}pez~Droguett, and Marcelo~Ramos Martins.
\newblock Towards interpretable deep learning: a feature selection framework for prognostics and health management using deep neural networks.
\newblock {\em Sensors}, 21(17):5888, 2021.

\bibitem{sanchez2007filter}
Noelia S{\'a}nchez-Maro{\~n}o, Amparo Alonso-Betanzos, and Mar{\'\i}a Tombilla-Sanrom{\'a}n.
\newblock Filter methods for feature selection--a comparative study.
\newblock In {\em International conference on intelligent data engineering and automated learning}, pages 178--187. Springer, 2007.

\bibitem{el2016review}
Naoual El~Aboudi and Laila Benhlima.
\newblock Review on wrapper feature selection approaches.
\newblock In {\em 2016 international conference on engineering \& MIS (ICEMIS)}, pages 1--5. IEEE, 2016.

\bibitem{wang2015embedded}
Suhang Wang, Jiliang Tang, and Huan Liu.
\newblock Embedded unsupervised feature selection.
\newblock In {\em Proceedings of the AAAI conference on artificial intelligence}, volume~29, 2015.

\bibitem{jovic2015review}
Alan Jovi{\'c}, Karla Brki{\'c}, and Nikola Bogunovi{\'c}.
\newblock A review of feature selection methods with applications.
\newblock In {\em 2015 38th international convention on information and communication technology, electronics and microelectronics (MIPRO)}, pages 1200--1205. Ieee, 2015.

\bibitem{chen2007enhanced}
Xue-wen Chen and Jong~Cheol Jeong.
\newblock Enhanced recursive feature elimination.
\newblock In {\em Sixth international conference on machine learning and applications (ICMLA 2007)}, pages 429--435. IEEE, 2007.

\bibitem{fonti2017feature}
Valeria Fonti and Eduard Belitser.
\newblock Feature selection using lasso.
\newblock {\em VU Amsterdam research paper in business analytics}, 30:1--25, 2017.

\bibitem{kursa2010boruta}
Miron~B Kursa, Aleksander Jankowski, and Witold~R Rudnicki.
\newblock Boruta--a system for feature selection.
\newblock {\em Fundamenta informaticae}, 101(4):271--285, 2010.

\bibitem{alfian2022predicting}
Ganjar Alfian, Muhammad Syafrudin, Imam Fahrurrozi, Norma~Latif Fitriyani, Fransiskus Tatas~Dwi Atmaji, Tri Widodo, Nurul Bahiyah, Filip Benes, and Jongtae Rhee.
\newblock Predicting breast cancer from risk factors using svm and extra-trees-based feature selection method.
\newblock {\em Computers}, 11(9):136, 2022.

\bibitem{candes2018panning}
Emmanuel Candes, Yingying Fan, Lucas Janson, and Jinchi Lv.
\newblock Panning for gold:‘model-x’knockoffs for high dimensional controlled variable selection.
\newblock {\em Journal of the Royal Statistical Society Series B: Statistical Methodology}, 80(3):551--577, 2018.

\bibitem{cohen2013statistical}
Jacob Cohen.
\newblock {\em Statistical power analysis for the behavioral sciences}.
\newblock routledge, 2013.

\bibitem{banerjee2009hypothesis}
Amitav Banerjee, Umesh~B Chitnis, Sudhir~L Jadhav, Jitendra~Shyamsundar Bhawalkar, and Suprakash Chaudhury.
\newblock Hypothesis testing, type i and type ii errors.
\newblock {\em Industrial psychiatry journal}, 18(2):127--131, 2009.

\bibitem{real1996probabilistic}
Raimundo Real and Juan~M Vargas.
\newblock The probabilistic basis of jaccard's index of similarity.
\newblock {\em Systematic biology}, 45(3):380--385, 1996.

\bibitem{golovenkin2020trajectories}
Sergey~E Golovenkin, Jonathan Bac, Alexander Chervov, Evgeny~M Mirkes, Yuliya~V Orlova, Emmanuel Barillot, Alexander~N Gorban, and Andrei Zinovyev.
\newblock Trajectories, bifurcations, and pseudo-time in large clinical datasets: applications to myocardial infarction and diabetes data.
\newblock {\em GigaScience}, 9(11):giaa128, 2020.

\bibitem{sakar2013collection}
Betul~Erdogdu Sakar, M~Erdem Isenkul, C~Okan Sakar, Ahmet Sertbas, Fikret Gurgen, Sakir Delil, Hulya Apaydin, and Olcay Kursun.
\newblock Collection and analysis of a parkinson speech dataset with multiple types of sound recordings.
\newblock {\em IEEE journal of biomedical and health informatics}, 17(4):828--834, 2013.

\bibitem{pal2022shockmodes}
Ridam Pal, Shaswat Patel, Akshala Bhatnagar, Hardik Garg, Pradeep Singh, Ritesh~Singh Soun, Aditya Agarwal, Aditya Nagori, Ashish Khanna, Rakesh Lodha, et~al.
\newblock Shockmodes: A multimodal model for prognosticating intensive care outcomes from physician notes and vitals.
\newblock {\em medRxiv}, pages 2022--12, 2022.

\bibitem{panda2024deciphering}
Gayatri Panda and Arjun Ray.
\newblock Deciphering cas9 specificity: Role of domain dynamics and rna: Dna hybrid interactions revealed through machine learning and accelerated molecular simulations.
\newblock {\em International Journal of Biological Macromolecules}, 283:137835, 2024.

\bibitem{awoyemi2017credit}
John~O Awoyemi, Adebayo~O Adetunmbi, and Samuel~A Oluwadare.
\newblock Credit card fraud detection using machine learning techniques: A comparative analysis.
\newblock In {\em 2017 international conference on computing networking and informatics (ICCNI)}, pages 1--9. IEEE, 2017.

\bibitem{robinson2024loan}
Natasha Robinson and Nidhi Sindhwani.
\newblock Loan default prediction using machine learning.
\newblock In {\em 2024 11th International Conference on Reliability, Infocom Technologies and Optimization (Trends and Future Directions)(ICRITO)}, pages 1--5. IEEE, 2024.

\bibitem{silva2008using}
Alice Silva.
\newblock Using data mining to predict secondary school student performance.
\newblock {\em Journal Name Placeholder}, 2008.

\bibitem{jethwa2024comprehensive}
Sejal Jethwa and Neha Vora.
\newblock A comprehensive analysis of airline passenger’s satisfaction through classification model.
\newblock {\em Journal of Trends in Computer Science and Smart Technology}, 6(4):324--337, 2024.

\end{thebibliography}
\vspace{1 cm}
\appendix
\begin{center}\textbf{\Large Appendix}\end{center}

\section{Proofs of Theoretical Results} \label{potr}

Let $f_j$ be a candidate feature. For bootstrap replicate $i=1,\dots,b$, define:
\begin{itemize}
    \item $I^{(i)}_j(l)$: the importance score of feature $f_j$ when evaluated with $l$ synthetic noise features included;
    \item $I^{(i)}(\varepsilon_k;l)$: the importance score of the $k$-th synthetic noise feature in replicate $i$ when $l$ noise features are included;
    \item $M^{(i)}(l) = \max_{1 \leq k \leq l} I^{(i)}(\varepsilon_k;l)$: the maximum noise importance at level $l$ in replicate $i$;
    \item $D_i(l) = I^{(i)}_j(l) - M^{(i)}(l)$: the paired difference between the signal feature and the maximum noise feature.
\end{itemize}

Let $T^+(l)$ be the one-sided Wilcoxon signed-rank statistic computed from $\{D_i(l)\}_{i=1}^b$, and let $c_\alpha(b)$ be the Wilcoxon critical value at level $\alpha$, which depends only on $b$. Assume:




\noindent\textbf{(A1) Monotone differences.}
$D_i(l+1)\le D_i(l)$ for all $i$. It suffices that
$I^{(i)}_j(l+1)\le I^{(i)}_j(l)$ and
$M^{(i)}(l+1)\ge M^{(i)}(l)$ hold pointwise.

\medskip

\noindent\textbf{(A2) Continuity.}
Feature importances are continuous (ties occur with probability zero).
In practice, ties use mid-ranks and zeros are dropped, with $b$
interpreted as the effective number of nonzero pairs.

\medskip

\noindent\textbf{(A3) Independence.}
The paired differences $\{D_i(l)\}_{i=1}^b$ obtained from bootstrap
replicates are treated as approximately independent draws from the
sampling distribution of the importance statistic.

\medskip

\noindent\textbf{(A0) Null exchangeability.}
Under $H_{0j}$, within each bootstrap replicate $i$, the candidate
importance $I^{(i)}_j(l)$ and the noise importances
$I^{(i)}(\varepsilon_1;l),\ldots,I^{(i)}(\varepsilon_l;l)$ are
exchangeable. This holds, in particular, when the benchmark is
constructed from permuted copies (shadow versions) of the candidate
feature, which share its marginal distribution while being
independent of $Y$ under the null.

\medskip

\noindent\textbf{(A1$'$) Supermartingale differences.}
$\mathbb{E}\!\left[D_i(l+1)\mid\mathcal{F}_l\right]\le D_i(l)$ for all
$i$, where $\mathcal{F}_l$ denotes the $\sigma$-field generated by the
construction through level $l$. This is implied by, but strictly
weaker than, the pointwise assumption (A1).

Assumption (A1) is motivated by the sequential nature of the construction. 
At step $l+1$, the data used at level $l$ are held fixed while an additional independent noise feature is appended. 
Consequently, the maximum noise importance $M^{(i)}(l)$ tends to increase (or remain stable) with $l$, as it is computed over an expanding set of noise variables. 
If total importance is bounded, without loss of generality, normalized such that $\sum_j I^{(i)}_j(l)=1$, then an increase in the importance of the newly added noise feature intuitively implies that the relative importance of some existing features is likely to decrease or remain unchanged. 
Hence, in expectation, the sequence of differences 
$D_i(l)=I^{(i)}_j(l)-M^{(i)}(l)$ can be viewed as nonincreasing in $l$, 
providing an intuitive justification for (A1).

\begin{lemma}[Wilcoxon Signed Rank (WSR) decomposition]\label{lem:wsr-decomp}
Let $\{D_i(l)\}_{i=1}^b$ denote the signed paired differences at level $l$, 
and let $T^+(l)$ denote the Wilcoxon positive rank sum statistic computed
from the absolute values $\{|D_1(l)|,\dots,|D_b(l)|\}$, 
after discarding zeros and assuming continuity (no ties).  
Define the index sets
\[
\begin{aligned}
P(l) &= \{\,i : D_i(l)>0\,\}, \qquad
N(l) = \{\,i : D_i(l)<0\,\}, \\
p(l) &= |P(l)|.
\end{aligned}
\]
Then $T^+(l)$ admits the decomposition
\begin{equation}
\begin{aligned}
T^+(l)
={}& \frac{p(l)\big(p(l)+1\big)}{2} \\
&+ \sum_{i\in P(l)}\sum_{j\in N(l)}
\mathbf{1}\!\left\{
|D_i(l)|>|D_j(l)|
\right\}.
\end{aligned}
\label{eq:tplus-decomposition}
\end{equation}
\end{lemma}

\begin{proof}
Fix $l$ and suppress $(l)$ in the notation for readability. Let
\[
P=\{i:D_i>0\},\quad N=\{i:D_i<0\},\quad p=|P|.
\]
By definition, $T^+=\sum_{i\in P} R_i$, where $R_i$ is the rank of $|D_i|$ among the nonzero absolute values $\{|D_1|,\dots,|D_b|\}$, ranked from smallest to largest, with no ties by assumption.

For any $i\in P$, the rank $R_i$ can be written as
\begin{equation}\label{eq:rank-expand}
\begin{aligned}
R_i ={}& 1
+ \#\{u\in P\setminus\{i\}: |D_u|<|D_i|\} \\
&+ \#\{v\in N: |D_v|<|D_i|\}.
\end{aligned}
\end{equation}
This is the standard “rank equals one plus the number of strictly smaller elements” identity, applied to the multiset of nonzero absolute values.

Summing \eqref{eq:rank-expand} over $i\in P$ yields three contributions:
\begin{equation}
\begin{aligned}
\sum_{i\in P} R_i
={}& \underbrace{\sum_{i\in P} 1}_{\text{(I)}} \\
&+ \underbrace{\sum_{i\in P}
\#\{u\in P\setminus\{i\}: |D_u|<|D_i|\}}_{\text{(II)}} \\
&+ \underbrace{\sum_{i\in P}
\#\{v\in N: |D_v|<|D_i|\}}_{\text{(III)}}.
\end{aligned}
\label{eq:rank-sum}
\end{equation}

\smallskip
\noindent\emph{Term (I):} $\sum_{i\in P} 1 = |P| = p$.

\smallskip
\noindent\emph{Term (II):} Among the $p$ positives, each unordered pair $\{i,u\}$ contributes exactly one to the sum because there are no ties. Hence the contribution is $\binom{p}{2}$.

\smallskip
\noindent\emph{Term (III):} By definition,
\[
\sum_{i\in P} \#\{v\in N: |D_v|<|D_i|\}
\;=\; \sum_{i\in P}\sum_{j\in N} \mathbf{1}\{|D_i|>|D_j|\}.
\]

Putting these together,
\begin{equation}
\begin{aligned}
T^+
&= p + \binom{p}{2}
+ \sum_{i\in P}\sum_{j\in N}
\mathbf{1}\{|D_i|>|D_j|\} \\
&= \frac{p(p+1)}{2}
+ \sum_{i\in P}\sum_{j\in N}
\mathbf{1}\{|D_i|>|D_j|\}.
\end{aligned}
\label{eq:tplus}
\end{equation}

This is the claimed identity after restoring the explicit dependence
on $l$.

\end{proof}

\subsection{Proof of Proposition 4.2 (Power is Decreasing in \texorpdfstring{$l$}{l} for the Wilcoxon Signed-Rank Test)}

Let $P(l) = \{i : D_i(l) > 0\}$ and $N(l) = \{i : D_i(l) < 0\}$, with $p(l)=|P(l)|$. Define
\[
T^+(l) \;=\; \sum_{i=1}^b R_i(l)\,\mathbf{1}\{D_i(l) > 0\},
\]
where $R_i(l)$ is the rank of $|D_i(l)|$ among the nonzero values $\{|D_1(l)|,\dots,|D_b(l)|\}$.

\begin{proposition}[No cross-over $\Rightarrow$ $T^+$ cannot increase]\label{prop:no-cross}
If $P(l+1)=P(l)$ and $N(l+1)=N(l)$, then $T^+(l+1)\le T^+(l)$.
\end{proposition}

\begin{proof}
From Lemma~\ref{lem:wsr-decomp},
\[
T^+(l) \;=\; \frac{p(l)(p(l)+1)}{2}
\;+\;\sum_{i\in P(l)}\sum_{j\in N(l)} \mathbf{1}\{|D_i(l)|>|D_j(l)|\}.
\]

\textbf{Case 1: Baseline term.}  
Since $P(l+1)=P(l)$, we have $p(l+1)=p(l)$, so the baseline term $\tfrac{p(l)(p(l)+1)}{2}$ is identical at $l$ and $l+1$.

\textbf{Case 2: Positive Negative (PN) comparisons.}  
For $i\in P(l)$, Assumption~(A1) implies $|D_i(l+1)|\le |D_i(l)|$.  
For $j\in N(l)$, Assumption~(A1) implies $|D_j(l+1)|\ge |D_j(l)|$.  
Hence for each $(i,j)\in P(l)\times N(l)$,
\[
\mathbf{1}\{|D_i(l+1)|>|D_j(l+1)|\}
\;\le\;
\mathbf{1}\{|D_i(l)|>|D_j(l)|\}.
\]
Thus every PN indicator is non-increasing.

\textbf{Conclusion.}  
The baseline is unchanged, and the PN sum cannot increase, so $T^+(l+1)\le T^+(l)$.
\end{proof}

\begin{proposition}[Any cross-over $\Rightarrow$ $T^+$ strictly decreases]\label{prop:with-cross}
If $P(l+1)\subsetneq P(l)$ (i.e., at least one index that is positive at level $l$ becomes non-positive at level $l+1$), then $T^+(l+1) < T^+(l)$.
\end{proposition}

\begin{proof}
Let $t\in P(l)$ with $D_t(l+1)\le 0$. Set
\[
p \coloneqq p(l)=|P(l)|,\qquad
c_t \coloneqq \sum_{j\in N(l)} \mathbf{1}\{|D_t(l)|>|D_j(l)|\}.
\]

\textbf{Case 1: Baseline term.}  
When $t$ leaves the positive set, the number of positives drops from $p$ to $p-1$. The baseline term decreases by
\[
\frac{(p-1)p}{2} - \frac{p(p+1)}{2} = -p.
\]

\textbf{Case 2: Loss of $t$'s PN wins.}  
In $T^+(l)$, $t$ contributes $c_t$ units by beating $c_t$ negatives. These are lost at $l+1$.

\textbf{Case 3: Possible recovery against $t$.}  
At $l+1$, $t$ is in $N(l+1)$ and can be beaten by remaining positives $i\in P(l)\setminus\{t\}$. Each such comparison contributes at most one unit, so the total possible recovery is at most $(p-1)$.

\textbf{Combine.}  
Therefore,
\[
T^+(l+1)-T^+(l) \;\le\; (-p) + (p-1) - c_t = -1 - c_t \;\le -1.
\]
Thus, a single cross-over strictly decreases $T^+$ by at least $1$. Multiple cross-overs decrease it further.
\end{proof}

\begin{corollary}[Power monotonicity]
Under (A1)–(A2), the WSR statistic satisfies $T^+(l+1)\le T^+(l)$ point-wise, with strict inequality when a crossover occurs. Since the null distribution of the Wilcoxon signed rank statistic depends only on the number of effective pairs $b$ (and not on $l$), the critical cutoff $c_\alpha$ is fixed across $l$, so the monotonicity of $T^+(l)$ directly implies monotonicity of the test power $\beta_j(l)$.
Hence the test power
\[
\beta_j(l) = \Pr\!\big(T^+(l)\ge c_\alpha \,\big|\, H_{1j}\big)
\]
is a non-increasing function of $l$.
\end{corollary}

\subsection{Proof of Proposition 4.1 (Type I Error Decreases in \texorpdfstring{$l$}{l} }

\begin{proposition}[Type I error decreases with $l$]\label{prop:type1-decreases}

Type~I error of the WSR test,
\[
\Pr_{H_0}\!\big(T^+(l)\ge c_\alpha(b)\big),
\]
is a non-increasing function of $l$, and strictly decreasing whenever cross-overs $D_i(l)>0\ge D_i(l+1)$ occur with positive probability.
\end{proposition}

\begin{proof}
By (A1), $D_i(l+1)\le D_i(l)$ for all $i$. Writing $P(l)=\{i:D_i(l)>0\}$, $N(l)=\{i:D_i(l)<0\}$, and $p(l)=|P(l)|$, the Wilcoxon decomposition (Lemma~\ref{lem:wsr-decomp}) gives
\[
T^+(l)\;=\;\frac{p(l)\big(p(l)+1\big)}{2}\;+\;\sum_{i\in P(l)}\sum_{j\in N(l)} \mathbf{1}\!\big\{|D_i(l)|>|D_j(l)|\big\}.
\]
If $P(l+1)=P(l)$ and $N(l+1)=N(l)$, then by (A1) every indicator in the double sum is non-increasing, while the baseline term $\tfrac{p(l)(p(l)+1)}{2}$ is unchanged, hence $T^+(l+1)\le T^+(l)$ (Proposition~\ref{prop:no-cross}). If some $t\in P(l)$ crosses to non-positive at $l+1$, then $T^+(l+1)<T^+(l)$ (Proposition~\ref{prop:with-cross}). Therefore,
\[
T^+(l+1)\ \le_{\text{a.s.}}\ T^+(l),
\]
with strict decrease on any realization exhibiting a cross-over.

Since the Wilcoxon critical value $c_\alpha(b)$ depends only on the number of effective nonzero pairs $b$ (and not on $l$), for any fixed $c$ we have
\[
\Pr_{H_0}\!\big(T^+(l+1)\ge c\big)\ \le\ \Pr_{H_0}\!\big(T^+(l)\ge c\big).
\]
Taking $c=c_\alpha(b)$ yields
\[
\Pr_{H_0}\!\big(T^+(l+1)\ge c_\alpha(b)\big)\ \le\ \Pr_{H_0}\!\big(T^+(l)\ge c_\alpha(b)\big),
\]
which proves that the Type~I error is non-increasing in $l$. Strict inequality holds whenever cross-overs occur with positive probability.
\end{proof}

\subsection{Exact Null Rate, Null Mean, and Finite-Sample Type~I Control}

\begin{proposition}[Exact null exceedance rate]\label{prop:exact-rate}
Under $H_{0j}$, (A0), and (A2), for each replicate $i$,
\[
\Pr\!\big(D_i(l)>0 \,\big|\, H_{0j}\big)\;=\;\frac{1}{l+1}.
\]
\end{proposition}

\begin{proof}
By definition $D_i(l)>0$ if and only if $I^{(i)}_j(l)>\max_{1\le k\le l} I^{(i)}(\varepsilon_k;l)$, that is, if and only if $I^{(i)}_j(l)$ is the strict maximum of the collection
\[
\mathcal{V}_i \;=\; \big\{\, I^{(i)}_j(l),\; I^{(i)}(\varepsilon_1;l),\dots, I^{(i)}(\varepsilon_l;l) \,\big\},
\]
which contains $l+1$ elements. By (A2), the elements of $\mathcal{V}_i$ are almost surely distinct, so exactly one attains the maximum. By (A0), the $l+1$ elements are exchangeable, hence each is equally likely to be that maximum. Therefore, the probability that the candidate attains the maximum equals $1/(l+1)$.
\end{proof}

\begin{proposition}[Null mean of the positive count]\label{prop:null-mean}
Let $S(l)=\sum_{i=1}^b \mathbf{1}\{D_i(l)>0\}$. Under the conditions of Proposition~\ref{prop:exact-rate},
\[
\mathbb{E}\!\big[S(l)\,\big|\,H_{0j}\big]\;=\;\frac{b}{l+1},
\]
regardless of any dependence among the replicates $\{D_i(l)\}_{i=1}^b$.
\end{proposition}

\begin{proof}
By linearity of expectation,
\[
\mathbb{E}\!\big[S(l)\big]=\sum_{i=1}^b \mathbb{E}\!\big[\mathbf{1}\{D_i(l)>0\}\big]=\sum_{i=1}^b \Pr\!\big(D_i(l)>0\big),
\]
and each term equals $1/(l+1)$ by Proposition~\ref{prop:exact-rate}. Linearity of expectation does not require independence, so dependence across replicates leaves the sum unchanged, giving $\mathbb{E}[S(l)]=b/(l+1)$.
\end{proof}

\begin{proposition}[Finite-sample Type~I control]\label{prop:finite-type1}
Adopt the signed-rank null model under which, conditional on the magnitudes $\{|D_i(l)|\}_{i=1}^b$, the sign indicators $\{\mathbf{1}\{D_i(l)>0\}\}_{i=1}^b$ are independent with common exceedance probability $q_l$, and let $c_\alpha(b)$ be the critical value calibrated at the symmetric value $q=1/2$. Under $H_{0j}$, (A0), and (A2), Proposition~\ref{prop:exact-rate} gives $q_l=1/(l+1)$, and:
\begin{enumerate}
\item[(i)] for $l=1$, $q_1=1/2$, so the differences $D_i(1)$ are symmetric about zero and the test attains exact level $\alpha$;
\item[(ii)] for $l\ge 2$, $q_l=1/(l+1)<1/2$, so $T^+(l)$ is stochastically dominated by its symmetric-null counterpart and the achieved Type~I error satisfies $\alpha_j(l)=\Pr_{H_0}\!\big(T^+(l)\ge c_\alpha(b)\big)\le\alpha$.
\end{enumerate}
\end{proposition}

\begin{proof}
(i) For $l=1$, $D_i(1)=I^{(i)}_j(1)-I^{(i)}(\varepsilon_1;1)$. By (A0) the pair $\big(I^{(i)}_j(1),\,I^{(i)}(\varepsilon_1;1)\big)$ is exchangeable, so $D_i(1)$ and $-D_i(1)$ have the same distribution; that is, $D_i(1)$ is symmetric about zero. Hence, conditional on $|D_i(1)|$, the sign is positive with probability $1/2$, which is exactly the model under which $c_\alpha(b)$ is defined. With independence across replicates (A3), the WSR test attains its nominal level $\alpha$.

(ii) Write $T^+(l)=\sum_{i=1}^b R_i(l)\,\mathbf{1}\{D_i(l)>0\}$, where the ranks $R_i(l)$ are nonnegative and determined solely by $\{|D_i(l)|\}$. Condition on $\{|D_i(l)|\}$, so the ranks are fixed. Let $S_i\sim\mathrm{Bernoulli}(q_l)$ (the true sign indicators, independent by the signed-rank model) and $S_i^{\mathrm{sym}}\sim\mathrm{Bernoulli}(1/2)$ (the symmetric-null indicators). Since $q_l\le 1/2$ for $l\ge 2$, one may couple the two families so that $S_i\le S_i^{\mathrm{sym}}$ almost surely for every $i$. Because the ranks $R_i(l)$ are nonnegative and identical in both models,
\[
T^+(l)=\sum_{i=1}^b R_i(l)\,S_i \;\le\; \sum_{i=1}^b R_i(l)\,S_i^{\mathrm{sym}}=T^{+}_{\mathrm{sym}}(l)
\]
holds almost surely under the coupling. Therefore $T^+(l)$ is stochastically dominated by $T^{+}_{\mathrm{sym}}(l)$, so $\Pr_{H_0}\!\big(T^+(l)\ge c\big)\le \Pr_{\mathrm{sym}}\!\big(T^+(l)\ge c\big)$ for every $c$. Taking $c=c_\alpha(b)$, and using that the symmetric-null tail at $c_\alpha(b)$ equals $\alpha$, gives $\alpha_j(l)\le\alpha$.
\end{proof}

\subsection{Monotonicity under the Weaker Assumption (A1\texorpdfstring{$'$}{'})}

\begin{proposition}[Expected-difference monotonicity]\label{prop:a1prime}
Under (A1$'$), for each $i$ the sequence $\{D_i(l)\}_l$ is a supermartingale with respect to $\{\mathcal{F}_l\}$, and consequently $\mathbb{E}[D_i(l)]$ is non-increasing in $l$.
\end{proposition}

\begin{proof}
By (A1$'$), $\mathbb{E}[D_i(l+1)\mid\mathcal{F}_l]\le D_i(l)$, which is the defining property of a supermartingale adapted to $\{\mathcal{F}_l\}$. Taking expectations and applying the tower property,
\[
\mathbb{E}[D_i(l+1)]=\mathbb{E}\!\big[\mathbb{E}[D_i(l+1)\mid\mathcal{F}_l]\big]\le \mathbb{E}[D_i(l)].
\]
Hence $\mathbb{E}[D_i(l)]$ is non-increasing in $l$.
\end{proof}

\begin{remark}
Assumption (A1) implies (A1$'$) but not conversely: (A1) requires the ordering $D_i(l+1)\le D_i(l)$ to hold pathwise in every replicate, whereas (A1$'$) requires it only in conditional mean. Propositions~\ref{prop:no-cross} and~\ref{prop:with-cross} use (A1) for their almost-sure statements, while (A1$'$) delivers the corresponding trends at the level of expectations. This provides a strictly weaker sufficient condition for the monotone behavior of the test in $l$.
\end{remark}

\subsection{Power Behavior, Consistency, and the Calibrated Sign Test}

Throughout this subsection, assume that under $H_{1j}$, within each replicate the candidate importance $I^{(i)}_j(l)$ is independent of the noise importances $I^{(i)}(\varepsilon_1;l),\dots,I^{(i)}(\varepsilon_l;l)$, which are independent and identically distributed with continuous distribution function $G_n$.

\begin{proposition}[Win probability under the alternative]\label{prop:winprob}
The per-replicate win probability satisfies
\[
\pi_l \;=\; \Pr\!\big(D_i(l)>0 \mid H_{1j}\big) \;=\; \mathbb{E}\!\big[\,G_n\big(I^{(i)}_j(l)\big)^{\,l}\,\big].
\]
If the law of $I^{(i)}_j(l)$ and $G_n$ do not vary with $l$, then $\pi_{l+1}\le\pi_l$ for all $l\ge 1$. If $I^{(i)}_j(l)=\mu_j$ almost surely, then $\pi_l=G_n(\mu_j)^{\,l}$.
\end{proposition}

\begin{proof}
Conditioning on $I^{(i)}_j(l)$ and using within-replicate independence together with the i.i.d.\ structure of the noise importances,
\begin{equation}
\begin{aligned}
\Pr\big(D_i(l)>0\big)
&= \Pr\Big(
\max_{1\le k\le l} I^{(i)}(\varepsilon_k;l)
< I^{(i)}_j(l)
\Big) \\
&= \mathbb{E}\Big[
G_n\big(I^{(i)}_j(l)\big)^{\,l}
\Big].
\end{aligned}
\label{eq:prob-positive}
\end{equation}
where continuity of $G_n$ ensures that ties occur with probability zero. Since $0\le G_n\le 1$, we have $G_n(t)^{\,l+1}\le G_n(t)^{\,l}$ pointwise; when the law of the candidate importance does not vary with $l$, taking expectations preserves the inequality, giving $\pi_{l+1}\le\pi_l$. The degenerate case follows by direct substitution.
\end{proof}

\begin{proposition}[Consistency]\label{prop:consistency}
Fix $l\ge 1$. Suppose that, as $n\to\infty$, $I^{(i)}_j(l)\xrightarrow{p}\mu_j$ for some $\mu_j>0$, and each noise importance converges to zero in probability, equivalently $G_n(\delta)\to 1$ for every $\delta>0$. Then $\pi_l\to 1$. Consequently, under (A3), for any $\alpha\ge (l+1)^{-b}$ the power of the exact sign test converges to one, and for any $\alpha\ge 2^{-b}$ the power of the one-sided WSR test converges to one.
\end{proposition}

\begin{proof}
By Proposition~\ref{prop:winprob} and monotonicity of $G_n$,
\[
\pi_l=\mathbb{E}\big[G_n(I^{(i)}_j(l))^{\,l}\big]\;\ge\;\Pr\big(I^{(i)}_j(l)\ge \mu_j/2\big)\, G_n(\mu_j/2)^{\,l}.
\]
The first factor converges to one because $I^{(i)}_j(l)\xrightarrow{p}\mu_j>\mu_j/2$, and the second converges to one because $G_n(\mu_j/2)\to 1$ and $l$ is fixed. Hence $\pi_l\to 1$.

Under (A3) the indicators $\mathbf{1}\{D_i(l)>0\}$ are independent across replicates, so $S(l)\sim\mathrm{Binomial}(b,\pi_l)$ and $\Pr(S(l)=b)=\pi_l^{\,b}\to 1$. When all differences are positive, the sign statistic attains its maximum $b$ and the WSR statistic attains its maximum $b(b+1)/2$. The exact sign test rejects at $S(l)=b$ whenever its null tail satisfies $\Pr_{H_0}\!\big(S(l)\ge b\big)=(l+1)^{-b}\le\alpha$, and the symmetric-null WSR test rejects at the maximal value whenever $2^{-b}\le\alpha$, since $\Pr_{\mathrm{sym}}\!\big(T^+=b(b+1)/2\big)=2^{-b}$. Under these non-degeneracy conditions, the rejection probability is bounded below by $\Pr(S(l)=b)\to 1$.
\end{proof}

\begin{proposition}[Exact calibration of the sign test]\label{prop:signtest}
Under $H_{0j}$, (A0), (A2), and (A3), $S(l)\sim \mathrm{Binomial}\!\big(b,\,1/(l+1)\big)$. The one-sided test that rejects when $S(l)\ge s_\alpha$, with $s_\alpha$ the smallest integer such that $\Pr\!\big(\mathrm{Binomial}(b,1/(l+1))\ge s_\alpha\big)\le\alpha$, has level at most $\alpha$, with exactness limited only by the discreteness of the Binomial distribution. Its null expectation equals $b/(l+1)$ even without (A3), by Proposition~\ref{prop:null-mean}.
\end{proposition}

\begin{proof}
By Proposition~\ref{prop:exact-rate}, each indicator $\mathbf{1}\{D_i(l)>0\}$ is Bernoulli with success probability $1/(l+1)$ under $H_{0j}$, (A0), and (A2). Under (A3), the indicators are independent across replicates, so their sum is Binomial with the stated parameters. The level statement is immediate from the definition of $s_\alpha$. The expectation statement was proved in Proposition~\ref{prop:null-mean} and requires no independence.
\end{proof}

\begin{remark}[Role of dependence across replicates]\label{rem:dependence}
Dependence among bootstrap replicates leaves the exact rate of Proposition~\ref{prop:exact-rate} and the null mean of Proposition~\ref{prop:null-mean} unchanged and affects only the dispersion of the test statistics. Conditional on the magnitudes $\{|D_i(l)|\}_{i=1}^b$, write $s_{(r)}$ for the sign indicator attached to rank $r$ and $q_r$ for its success probability. Then
\begin{equation}
\begin{aligned}
\operatorname{Var}\!\big(T^+(l)\big)
={}& \sum_{r=1}^{b} r^2\,q_r(1-q_r) \\
&+ 2\!\!\sum_{1\le r<r'\le b}\!\!
r\,r'\,\operatorname{Cov}\!\big(s_{(r)},s_{(r')}\big).
\end{aligned}
\label{eq:variance-tplus}
\end{equation}
which reduces to $\sum_r r^2/4=b(b+1)(2b+1)/24$ under the symmetric independent benchmark. Positive covariances induced by shared bootstrap observations inflate this variance. Assumption (A3) is therefore required only for the distributional calibration of $T^+(l)$ and $S(l)$; the exact rate and null-mean results above hold without it.
\end{remark}

\section{Experimental Design}\label{sec:experiments1}
\subsection{Simulation Studies} \label{sec:Simulation Studies}
We also synthesized data mimicking real use cases, conditioning on boundary and edge cases. While simulating the data, we considered the inclusion of factors such as sparsity, collinearity, and noise, which can allow for the systematic testing of the proposed approach across intricate settings. Hence, to systematically evaluate the performance of the proposed Noise-Augmented Bootstrap Feature Selection (NABFS) method, we conducted a comprehensive Monte Carlo simulation study against established feature selection techniques, which included Boruta \cite{kursa2010boruta} and Model-X Knockoffs \cite{candes2018panning}. The study was designed to assess the reliability and efficiency of each method under controlled yet diverse data-generating scenarios.

We simulated data from a logistic regression model with a compound symmetric correlation structure, a common framework for assessing feature selection performance in high-dimensional settings. Specifically, we designed a matrix $X \in \mathbb{R}^{n \times p}$, where each observation vector $X_i$ was drawn as
\[
X_i = \sqrt{1-\rho}\,Z_i + \sqrt{\rho}\,C_i\mathbf{1}_p,
\]
with $Z_i \sim N(0, I_p)$ representing independent noise and $C \sim N(0,1)$ denoting a shared latent factor introducing correlation among features. The true coefficient vector $\beta \in \mathbb{R}^p$ was defined such that the first $k$ features were true signals, each assigned a coefficient sampled uniformly from the interval $[-2, 2]$, while the remaining $p - k$ coefficients were set to zero. Binary outcomes were generated according to a logistic model,
\[
\Pr(Y_i = 1 \mid X_i) = \frac{1}{1 + \exp(-X_i^\top \beta)}.
\]

The experimental setup considered 50 candidate features ($p=50$), of which 20 were true signals ($k=20$), across sample sizes $n \in \{500, 1000, 3000\}$ in the $n>p$ regime. We additionally considered the $p>n$ regime with $p=900$, $k=20$, and $n \in \{300, 500, 800\}$. The correlation parameter $\rho$ was varied across different levels, allowing systematic examination of the effect of feature correlation on selection stability. For NABFS, we varied the number of synthetic noise features ($l \in \{1, 2, 3, 4, 5\}$) to study its robustness across different augmentation intensities.

Each configuration was repeated 20 Monte Carlo replicates to ensure statistical reliability and stability assessment. 

For every combination of $(n, \rho)$, the data generation process, feature selection technique, and performance evaluation were repeated 20 times, and the averaged results across replicates were reported. The summary metrics with corresponding parameters $\{n, \rho, p, k\}$, which included average power, Type I error, and Jaccard index, were compiled and reported below.

This design enables a rigorous and reproducible comparison of NABFS with Boruta and Model-X Knockoffs across varying sample sizes, correlation structures, and noise conditions, providing a comprehensive understanding of the methods’ stability, efficiency, and accuracy.

\subsection{Real Data Application} \label{sec:Real Data Application}

We have tested our designed architecture for feature selection across diverse datasets to establish the generalizability of the technique. We curated datasets related to healthcare, genomics, proteomics, finance, and weather to effectively validate the robustness of our technique. In the previous section, we demonstrated different scenarios in which we simulated the data based on various scenarios and tested the effectiveness of the model. 

Some of the real-world datasets related to healthcare were curated from the UCI Machine Learning Repository to understand the competence of the feature selection technique. Clinical data describing the spectrum of complications following myocardial infarction, contributed by Golovenkin et al. \cite{golovenkin2020trajectories}, presented a rich domain-relevant attribute set suitable for rigorous medical data interrogation. The Parkinson’s Disease Classification Dataset, made available by Sakar et al. \cite{sakar2013collection}, offers extensive biomedical voice measurements from individuals diagnosed with Parkinson’s disease, enabling exploratory analysis of high-dimensional biological signals through advanced model selection methodologies. ShockModes dataset, introduced by Ridam Pal et al. \cite{pal2022shockmodes}, is derived from the MIMIC-III intensive care unit database and comprises 17,294 ICU stays that were scored for the Shock Index (SI), a clinical metric used to prognosticate outcomes in critical care and emergency settings. The dataset identifies episodes of abnormal SI defined by sustained periods of elevated SI ($> 0.7$ for over 30 minutes) following a preceding window of normal SI, resulting in a curated cohort of 337 normal and 84 abnormal SI instances after exclusion criteria. The paper constituted of of three cohorts, of which two cohorts, the multi-modal cohort and the uni-modal cohort (comprised of unstructured clinical notes embedding) were used for testing the effectivity of our method. 

We even tested our dataset on molecular simulation data provided in the paper "Decrypting the mechanistic basis of CRISPR/Cas9 protein". This CRISPR/Cas9 study performed by Panda et al. \cite{panda2024deciphering} used machine learning to classify metastable heteroduplex states in SpCas9 and FnCas9 complexes based on structural features derived from molecular dynamics simulations. The heteroduplex-state dataset contained 57,074 frames for SpCas9 and 52,080 for FnCas9, generated from long-timescale MD trajectories of both on-target and off-target systems. Each frame was labeled using a Markov State Model that identified discrete structural macrostates based on the dynamics of the RNA:DNA hybrid. For every frame, structural features were extracted, covering both RNA:DNA hybrid geometry (rotational and translational parameters for base pairs 17–20 relative to PAM, interfacial contacts, and interaction energy) and protein structural properties (backbone RMSD, radius of gyration, SASA, key angular descriptors, and inter-domain contacts). The dataset was organized to support classification using three input modalities: RNA:DNA hybrid-only, protein-only, and combined features, enabling systematic comparison of their contributions to distinguishing metastable states of the RNA-DNA heteroduplex.

For the financial dataset, we used data from Kaggle as an open resource. The Credit Card Fraud Detection dataset \cite{awoyemi2017credit}, curated by the Machine Learning Group of ULB and available via the Kaggle data platform, comprises anonymized transactions made by European cardholders over a two-day period in September 2013. The Loan Default Prediction Dataset \cite{robinson2024loan}, sourced from a Coursera challenge and made publicly available on Kaggle, includes detailed financial, demographic, and behavioral attributes describing borrowers’ profiles and loan histories. 

The Student Performance Dataset, generated by Cortez \cite{silva2008using}, encompasses multidimensional academic, demographic, socio-economic, and institutional variables, providing a robust framework to benchmark feature selection approaches within educational research. The Airlines Customer Satisfaction dataset \cite{jethwa2024comprehensive} comprises over 120,000 passenger records aimed at predicting overall satisfaction with airline services. It serves as a comprehensive benchmark for evaluating factors that influence passenger contentment and loyalty.

The collation of datasets from real-world settings, both from open-source and private repositories, enables systematic and reproducible evaluation of our robust method (NABFS), offering insights into the robustness and adaptability of the proposed method under noise- and artifact-prone scenarios. 

\subsection{Metric Evaluation}

The performance of the proposed method was systematically evaluated using a comprehensive set of metrics tailored for both simulated and real-world data. For simulation studies, model efficacy was assessed using \textit{Power} \cite{cohen2013statistical}, \textit{Type~I Error} \cite{banerjee2009hypothesis}, and the \textit{Jaccard Index} \cite{real1996probabilistic}, which collectively capture sensitivity, specificity, and selection stability. For real-world datasets, performance was further validated using the \textit{F1-score} and \textit{AUC-score}, reflecting the classification accuracy and discriminative capability of models trained on the selected feature subsets. A brief description of each metric is provided below.

\begin{itemize}
\item \textbf{Power:} Measures the proportion of true signal features correctly identified, indicating the model’s ability to detect relevant variables.
\item \textbf{Type~I error:} Represents the proportion of irrelevant (noise) features incorrectly selected, quantifying the model’s resistance to overfitting.
\item \textbf{Jaccard index:} Evaluates the overlap between the selected and true feature sets, summarizing the overall selection accuracy and stability.
\item \textbf{F1-score:} The harmonic mean of precision and recall, offering a balanced measure of classification performance, particularly in imbalanced datasets.
\item \textbf{AUC-score:} The area under the ROC curve, providing a threshold-independent evaluation of the model’s discriminative strength.
\end{itemize}

Together, these metrics provide a holistic assessment of the proposed method, capturing its statistical reliability in controlled simulations and its predictive robustness on real-world datasets.

\subsection{Time Complexity Analysis}

To characterize the computational cost of the proposed noise-based feature selection procedure, we analyze each component of the algorithm as a function of the dataset dimensions and hyperparameters. Let
$n$ denote the number of samples,
$p$ the number of original (non-noise) features,
$m$ the number of injected noise features, and
$B$ the number of bootstrap iterations.
The total working dimensionality is therefore $d = p + m$.

\paragraph{1. Noise Augmentation and Scaling:}
The augmentation step introduces $m$ synthetic noise variables sampled independently across $n$ observations, incurring a cost of
\[
O(nm).
\]
MinMax scaling over all $d$ features then contributes
\[
O(nd).
\]
Since $d = p+m$, both operations are linear in the dataset size and negligible relative to the bootstrap model-fitting stage.

\paragraph{2. Bootstrap Resampling:}
Each of the $B$ bootstrap iterations draws a resampled dataset of size $n$, requiring
\[
O(n)
\]
per iteration and
\[
O(Bn)
\]
overall. This step is dominated by the subsequent model-fitting phase.

\paragraph{3. Model Fitting per Bootstrap (Logistic Regression and XGBoost):}
Model fitting is the dominant component of NABFS. We analyze the computational cost for two model classes used in our experiments.

\textbf{Logistic Regression.}  
In each bootstrap iteration, an L2-regularized logistic regression with $d$ predictors is fitted using LBFGS, where each optimization step scales as
\[
O(nd),
\]
and typically converges in a constant number $k$ of iterations. Thus, each model fit costs
\[
O(knd) \approx O(nd),
\]
yielding
\[
O(Bnd)
\]
across all bootstraps.

\textbf{XGBoost.}  
For XGBoost with $T$ boosting rounds and maximum tree depth $h$, each iteration builds decision trees via repeated split evaluations, leading to a per-bootstrap complexity of
\[
O(T n h \log n).
\]
For fixed $T$ and $h$ (typical in practice), this reduces to
\[
O(n \log n),
\]
and therefore
\[
O(B n \log n)
\]
across all bootstrap iterations.

\paragraph{4. Extraction of Feature Importance:}
For each iteration, coefficient magnitudes or tree-based importance values for all $d$ features are recorded:
\[
O(d)
\]
per iteration and
\[
O(Bd)
\]
overall negligible compared to model fitting.

\paragraph{5. Hypothesis Testing via Wilcoxon Signed-Rank Test:}
For each of the $p$ real features, a paired Wilcoxon signed-rank test is performed using $B$ bootstrap samples. Each test costs
\[
O(B \log B),
\]
leading to
\[
O(p B \log B)
\]
overall. This becomes relevant only when $p$ is very large.

\paragraph{Total Time Complexity:}
Combining all components, the overall computational cost is dominated by the model-fitting term. For logistic regression, the total complexity is
\[
T_{\mathrm{LR}}(n,p,m,B) = O\!\left(Bn(p+m)\right) + O\!\left(p B \log B\right).
\]
For XGBoost, the total complexity becomes
\[
T_{\mathrm{XGB}}(n,p,m,B) = O\!\left(B n \log n\right) + O\!\left(p B \log B\right).
\]

\paragraph{Interpretation:}
The bootstrap count $B$ is the primary driver of runtime. Logistic regression scales linearly in feature dimensionality, making it efficient when $d$ is moderate. XGBoost introduces a $\log n$ factor but is more robust at high dimensionality since its tree-based splits do not depend linearly on $d$. Noise dimensionality $m$ increases cost linearly, while sample size $n$ contributes multiplicatively with $B$. Overall, NABFS remains tractable for moderate to high-dimensional datasets across both model classes, with computational cost dominated by the chosen predictive model rather than the statistical testing procedure.

\subsection{Space Complexity Analysis}

The memory footprint of the proposed procedure is governed primarily by three quantities: the sample size $N$, the number of original features $d$, and the number of injected noise features $n_{\text{noise}}$. The augmented design matrix therefore requires storage proportional to
\[
S_{\text{data}} = O\!\left( N(d + n_{\text{noise}}) \right).
\]

Bootstrap resampling does not duplicate the underlying data; instead, each bootstrap iteration stores only an index vector of length $N$, giving a memory requirement of
\[
S_{\text{bootstrap}} = O(bN),
\]
for $b$ bootstrap iterations.

The coefficient trajectories for all features across all bootstraps contribute an additional
\[
S_{\text{coeff}} = O\!\left( b(d + n_{\text{noise}}) \right).
\]

The Wilcoxon signed-rank testing phase requires temporary arrays of size $O(b)$ for each feature, but these buffers are reused sequentially and therefore do not increase the asymptotic space requirement.

Combining all components, the total space usage is
\[
S_{\text{total}} = O\!\left( N(d + n_{\text{noise}}) + bN + b(d + n_{\text{noise}}) \right),
\]
which simplifies to
\[
S_{\text{total}} \approx O\!\left( N(d + n_{\text{noise}}) + b(d + N) \right).
\]

Thus, the memory usage is dominated by the augmented feature matrix and the coefficient histories accumulated across bootstrap iterations.

\section{Detailed Analysis of Results}
To evaluate the performance of the proposed Noise-Augmented Bootstrap Feature Selection (NABFS) method, we developed a controlled simulation framework based on a logistic regression model with a compound-symmetric correlation structure for data synthesis. Two regimes were examined, one where the sample size exceeded the number of features ($n > p$) and another reflecting the high-dimensional setting ($p > n$). The simulations systematically varied the correlation parameter ($\rho$) to examine the influence of inter-feature dependence on selection performance. Each configuration was repeated across 20 Monte Carlo replicates to ensure statistical reliability, and results were averaged across runs. Comparative analyses were performed against two established baseline methods, Boruta and Model-X Knockoffs, across three sample sizes ($n = 300, 500,$ and $800$). The evaluation metrics included average Power, Type I Error, and Jaccard Index, as summarized in Figure \ref{fig1}.

\begin{figure*}[t]
\centering
\includegraphics[width=0.9\textwidth]{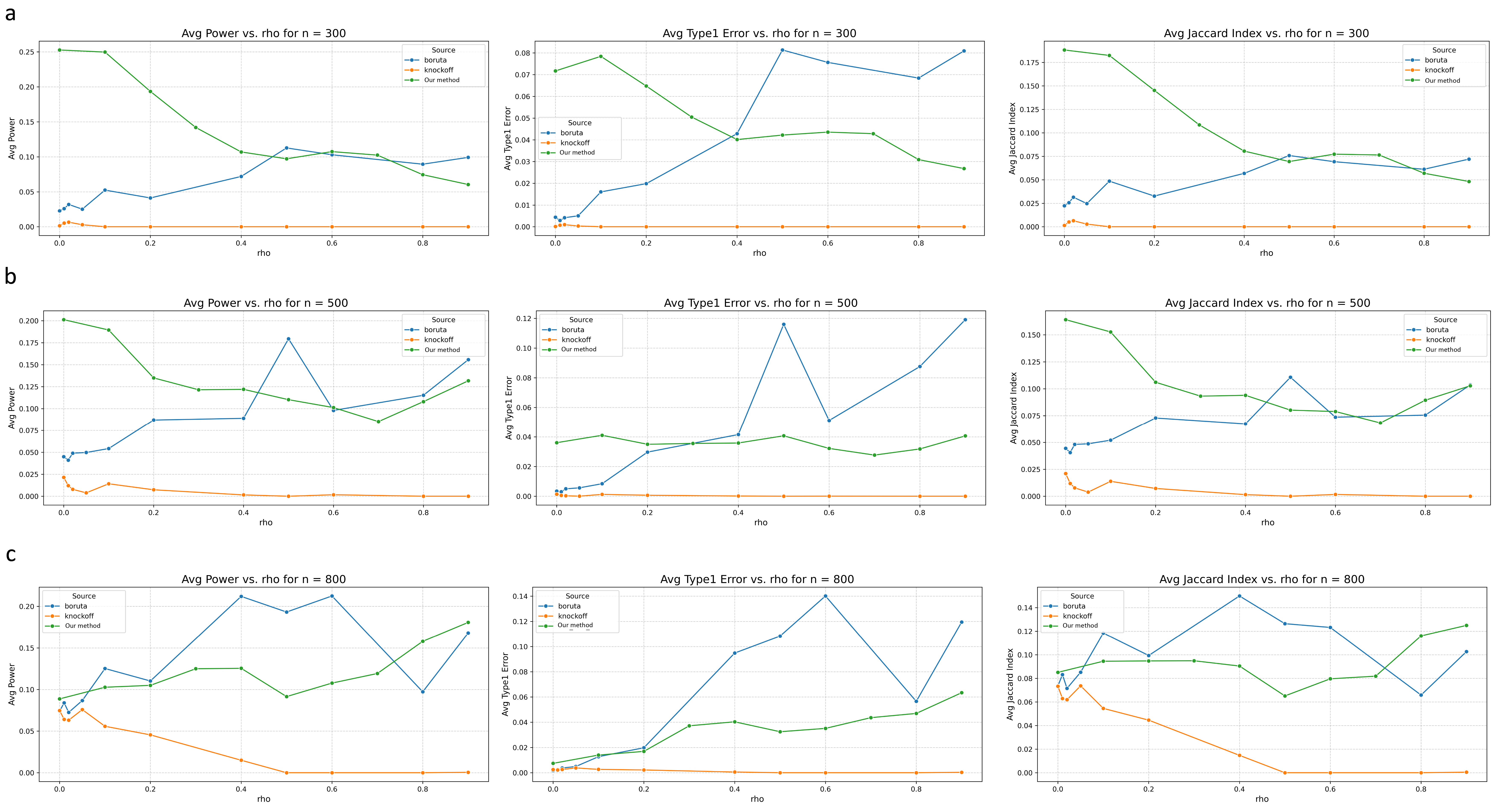} 
\caption{Performance of NABFS under the \(p > n\) setting, evaluated using Power, Type~I Error, and Jaccard Index across varying correlation levels.}
\label{fig1}
\end{figure*}

\begin{figure*}[t]
\centering
\includegraphics[width=0.9\textwidth]{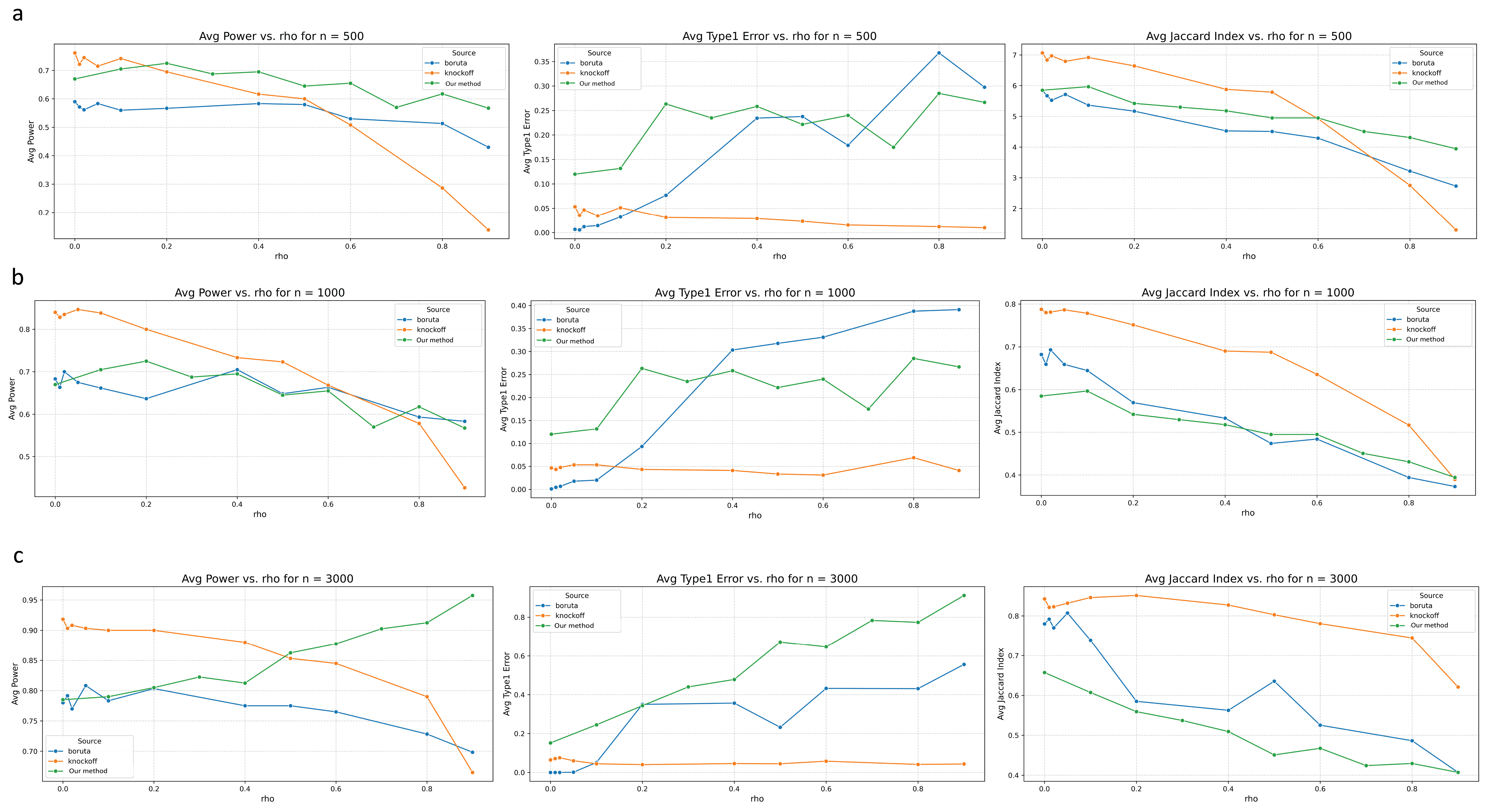} 
\caption{Performance of NABFS under the \(n > p\) setting, evaluated using Power, Type~I Error, and Jaccard Index across varying correlation levels.}
\label{fig52}
\end{figure*}

Our experiments in the high-dimensional regime ($p > n$) were conducted with a fixed feature space of $p = 900$, comprising 700 true signal features and 200 synthetic noise features. Under this setting, for \textbf{$n = 300$}, NABFS demonstrates consistently higher average power under low-to-moderate correlation, peaking near $0.25$ at $\rho = 0$ and declining toward $0.08$ as $\rho$ increases. This indicates the method’s strong sensitivity in identifying true signal features even in limited data samples. Boruta shows gradual improvement with increasing correlation, rising from about $0.02$ to $0.11$, while Knockoff remains largely conservative with negligible gain in power, staying close to zero throughout. The Type I error trends reveal that NABFS maintains moderate false-positive rates (about $0.03$ to $0.08$), achieving a balanced trade-off between sensitivity and specificity. The Jaccard Index further confirms NABFS’s robustness, with the highest overlap at low correlation (about $0.19$ at $\rho = 0$), remaining comparable to the baselines as correlation grows. As the sample size increases to \textbf{$n = 500$}, the advantage of NABFS becomes more pronounced. The average power curve stabilizes at higher levels (about $0.19$ at $\rho = 0$), reflecting improved statistical efficiency with larger data. Both Boruta and Knockoff show marginal improvement; however, Boruta continues to exhibit inflated Type I error, rising to about $0.12$ under strong correlation, while Knockoff remains conservative and fails to recover true features effectively, holding power near $0$. NABFS sustains its equilibrium between power and false discovery, keeping Type I error near the nominal level (about $0.03$ to $0.10$) and validating its scalability across varying data regimes. The Jaccard Index improves across all methods, but NABFS continues to achieve the highest stability and selection accuracy (about $0.14$ to $0.17$). For \textbf{$n = 800$}, NABFS maintains stable power across correlation levels (about $0.09$ to $0.16$) while keeping Type I error tightly controlled (about $0.01$ to $0.05$), with minimal increase despite higher data dimensionality. In contrast, Boruta attains higher power over part of the correlation range (peaking near $0.21$) but at the cost of the highest Type I error, rising to about $0.14$ under strong correlations ($\rho > 0.6$), while Knockoff continues to underperform in terms of power, collapsing toward zero as correlation grows. The Jaccard Index trends reaffirm NABFS’s robustness and generalizability, remaining close to Boruta and clearly above Knockoff, which decays to zero under strong correlation. Overall, across all sample sizes and correlation levels, NABFS achieves a robust balance between sensitivity and specificity, delivering the strongest power at small sample sizes and the most consistent control of Type I error. Its stable performance with increasing sample size emphasizes scalability and robustness, validating NABFS as a reliable feature selection technique for high-dimensional data with correlated and noisy features.

To evaluate NABFS in the \(n > p\) setting, we considered a simulation design with \(p = 50\) candidate features, of which \(k = 20\) were true signals, and sample sizes varied across \(n \in \{500, 1000, 3000\}\). This configuration enabled a controlled examination of how increasing data availability influences the performance and stability of the feature selection technique.

As shown in Figure \ref{fig52}, NABFS exhibits a consistent increase in average power with growing sample size, indicating its ability to accurately recover true signal features with data abundance. At $n = 500$, NABFS achieves moderate power ($\approx 0.65$ at $\rho = 0$) with well-controlled Type I error (near the nominal level and close to zero for $\rho \ge 0.4$), maintaining balanced sensitivity and specificity. Unlike the baselines, its power stays almost flat across correlation, whereas Knockoff drops sharply from about $0.76$ to $0.29$ and Boruta declines from about $0.68$ to $0.51$; the Jaccard Index for NABFS likewise remains the most stable (about $0.61$ to $0.65$), demonstrating robustness against multicollinearity. At $n = 1000$, NABFS’s performance improves, with average power around $0.67$ to $0.73$ at low-to-moderate correlation and all three methods converging toward about $0.58$ at high correlation. Type I error remains moderate for NABFS (about $0.12$ to $0.29$), while Knockoff attains the lowest values (about $0.03$ to $0.12$) and Boruta shows the largest inflation, rising to about $0.39$. The Jaccard Index trends mirror the power trajectory, with Knockoff highest (about $0.79$ falling to $0.52$) and NABFS declining more gently from about $0.59$ to $0.43$ as sample dependencies grow. At the largest sample size, $n = 3000$, NABFS shows the clearest benefit of data abundance, with power rising steadily from about $0.76$ at $\rho = 0$ to about $0.91$ at $\rho = 0.8$, and a Jaccard Index of about $0.66$ at $\rho = 0$. Knockoff remains competitive on power (about $0.79$ to $0.92$) and retains the highest Jaccard Index (about $0.84$ falling to $0.75$) with the lowest, flattest Type I error (about $0.05$). Even under stronger correlation regimes NABFS maintains competitive power, though with some Type I error inflation at high correlation, highlighting its resilience in scenarios when correlated features are more. These results collectively affirm the robustness of NABFS in scenarios where data abundance enhances feature retention capacity. Overall, across all configurations, NABFS demonstrates a desirable trade-off between sensitivity and specificity, with consistent improvement in selection stability as data samples ($n$) increase. The convergence of Power and Jaccard Index values in the large-sample regime substantiates the theoretical underpinnings of NABFS, validating the efficacy of our technique.

Across the real-world datasets, NABFS retained substantially fewer features than most baselines while remaining competitive on the imbalance-aware and regression metrics. Tables~\ref{tab:supp_binary},~\ref{tab:supp_multiclass}, and~\ref{tab:supp_reg} report the full comparison against Boruta, RFE, LASSO, Elastic Net, mRMR, and Model-X Knockoffs on the binary, multiclass, and regression datasets, respectively. For each classification dataset, we report the minority-class F1, the Matthews correlation coefficient (MCC), the area under the precision-recall curve (PR-AUC), and, for the multiclass datasets, balanced accuracy; for regression, we report RMSE and $R^2$. In every case, we also report the percentage of retained features.

On the clinical datasets, Parkinson's Disease Speech and Septic Shock, NABFS retained compact and interpretable feature subsets while remaining competitive with the baselines, and this behavior extended to the financial, biological, and tabular datasets as well. On the binary datasets (Table~\ref{tab:supp_binary}), NABFS attained the best minority-class F1, MCC, and PR-AUC among the probe- and wrapper-based baselines (Boruta, RFE, Elastic Net, and Model-X Knockoffs) on Credit Card Fraud while retaining only 24.1\% of the features, with the filter method mRMR remaining competitive at similar sparsity; it also improved minority-class F1 and MCC over Boruta and Model-X Knockoffs on the highly imbalanced Septic Shock cohort using a small fraction of the features. On Airline Customer Satisfaction, all methods performed near-identically. On Parkinson's Disease Speech, NABFS selected an extremely compact subset (two features) at a cost in MCC, consistent with the stringency--power trade-off established in Section~\ref{potr}: increasing the noise pool raises the empirical null threshold and can suppress weak signals. On the multiclass CRISPR off-target datasets (Table~\ref{tab:supp_multiclass}), NABFS matched the best baselines in MCC and balanced accuracy on SP EMX1 while retaining almost all informative features, and was competitive on FN EMX1. For regression (Table~\ref{tab:supp_reg}), NABFS attained the highest $R^2$ on Food Delivery Times and remained within $0.01$ in $R^2$ of the best baseline on the remaining datasets, typically using far fewer features than LASSO, Elastic Net, and Model-X Knockoffs.

\begin{table}[htbp]
\centering
\footnotesize
\setlength{\tabcolsep}{4pt}
\begin{tabular}{ll}
\hline
\textbf{Component} & \textbf{Version} \\
\hline
Operating system
& Ubuntu 24.04.4 LTS (Noble Numbat) \\

Kernel
& 7.0.0-28-generic, x86\_64 \\

Python
& 3.12.3 (GCC 13.3.0) \\

NumPy
& 2.3.2 \\

SciPy
& 1.16.1 \\

pandas
& 2.3.2 \\

scikit-learn
& 1.7.1 \\

XGBoost
& 3.0.4 \\

Boruta (BorutaPy)
& 0.4.3 \\

imbalanced-learn
& 0.14.0 \\

matplotlib
& 3.10.5 \\

joblib
& 1.5.1 \\

tqdm
& 4.67.1 \\

Notebook runtime
& JupyterLab 4.4.6, ipykernel 6.30.1 \\
\hline
\end{tabular}
\caption{Software environment used for the experiments.}
\label{tab:software}
\end{table}

\begin{table*}[t]
\centering
\renewcommand{\arraystretch}{1.15}
\setlength{\tabcolsep}{4pt}
\resizebox{\textwidth}{!}{%
\begin{tabular}{llccccccc}
\hline
\textbf{Dataset} & \textbf{Metric}
& \textbf{NABFS} & \textbf{Boruta} & \textbf{RFE}
& \textbf{LASSO} & \textbf{ElasticNet} & \textbf{mRMR} & \textbf{Knockoff} \\
\hline

\multirow{3}{*}{Food Delivery Times}
& \% Ret.
& 28.6 & 14.3 & 14.3 & 100.0 & 100.0 & 21.4 & 100.0 \\
& RMSE
& 12.200 & 12.404 & 12.404 & 12.362 & 12.362 & 12.427 & 12.362 \\
& $R^2$
& 0.648 & 0.636 & 0.636 & 0.638 & 0.638 & 0.635 & 0.638 \\
\hline

\multirow{3}{*}{Graduate Admission}
& \% Ret.
& 85.7 & 42.9 & 42.9 & 85.7 & 100.0 & 28.6 & 100.0 \\
& RMSE
& 0.072 & 0.073 & 0.073 & 0.071 & 0.071 & 0.073 & 0.071 \\
& $R^2$
& 0.773 & 0.767 & 0.767 & 0.780 & 0.781 & 0.763 & 0.781 \\
\hline

\multirow{3}{*}{Social Anxiety}
& \% Ret.
& 23.3 & 30.0 & 30.0 & 90.0 & 96.7 & 16.7 & 56.7 \\
& RMSE
& 1.101 & 1.099 & 1.099 & 1.099 & 1.099 & 1.101 & 1.099 \\
& $R^2$
& 0.737 & 0.738 & 0.738 & 0.738 & 0.738 & 0.737 & 0.738 \\
\hline

\multirow{3}{*}{Student Performance}
& \% Ret.
& 20.5 & 12.8 & 12.8 & 92.3 & 92.3 & 15.4 & 100.0 \\
& RMSE
& 3.534 & 3.524 & 3.569 & 3.540 & 3.540 & 3.671 & 3.540 \\
& $R^2$
& 0.158 & 0.163 & 0.141 & 0.156 & 0.156 & 0.092 & 0.155 \\
\hline

\end{tabular}%
}
\caption{Full comparison on the regression datasets. \% Ret.\ is the percentage of retained features.}
\label{tab:supp_reg}
\end{table*}

\begin{table*}[t]
\centering
\renewcommand{\arraystretch}{1.15}
\setlength{\tabcolsep}{4pt}
\resizebox{\textwidth}{!}{%
\begin{tabular}{llccccccc}
\hline
\textbf{Dataset} & \textbf{Metric}
& \textbf{NABFS} & \textbf{Boruta} & \textbf{RFE}
& \textbf{LASSO} & \textbf{ElasticNet} & \textbf{mRMR} & \textbf{Knockoff} \\
\hline
\multirow{3}{*}{FN EMX1 (CRISPR)}
& \% Ret.
& 83.3 & 85.7 & 85.7 & 100.0 & 100.0 & 14.3 & 71.4 \\
& MCC
& 0.519 & 0.519 & 0.519 & 0.519 & 0.519 & 0.533 & 0.540 \\
& Bal.\ Acc.
& 0.495 & 0.495 & 0.495 & 0.495 & 0.495 & 0.504 & 0.513 \\
\hline
\multirow{3}{*}{SP EMX1 (CRISPR)}
& \% Ret.
& 96.0 & 100.0 & 100.0 & 100.0 & 100.0 & 20.0 & 88.0 \\
& MCC
& 0.842 & 0.842 & 0.842 & 0.842 & 0.842 & 0.724 & 0.842 \\
& Bal.\ Acc.
& 0.685 & 0.685 & 0.685 & 0.685 & 0.685 & 0.639 & 0.685 \\
\hline
\end{tabular}%
}
\caption{Comparison on the multiclass CRISPR off-target datasets. Each method is reported by feature retention (\% Ret.), Matthews correlation coefficient (MCC), and balanced accuracy (Bal.\ Acc.). PR-AUC is omitted because it is not defined identically across the multiclass configurations for all methods.}
\label{tab:supp_multiclass}
\end{table*}

\begin{table*}[t]
\centering
\renewcommand{\arraystretch}{1.15}
\setlength{\tabcolsep}{4pt}
\resizebox{\textwidth}{!}{%
\begin{tabular}{llccccccc}
\hline
\textbf{Dataset} & \textbf{Metric}
& \textbf{NABFS} & \textbf{Boruta} & \textbf{RFE}
& \textbf{LASSO} & \textbf{ElasticNet} & \textbf{mRMR} & \textbf{Knockoff} \\
\hline

\multirow{4}{*}{Airline Cust. Sat.}
& \% Ret.
& 100.0 & 100.0 & 100.0 & 100.0 & 100.0 & 19.0 & 95.2 \\
& Min-F1
& 0.834 & 0.834 & 0.834 & 0.834 & 0.834 & 0.820 & 0.834 \\
& MCC
& 0.624 & 0.624 & 0.624 & 0.624 & 0.624 & 0.604 & 0.624 \\
& PR-AUC
& 0.926 & 0.926 & 0.926 & 0.926 & 0.926 & 0.901 & 0.926 \\
\hline

\multirow{4}{*}{Credit Card Fraud}
& \% Ret.
& 24.1 & 93.1 & 93.1 & 58.6 & 100.0 & 17.2 & 100.0 \\
& Min-F1
& 0.619 & 0.585 & 0.585 & 0.615 & 0.585 & 0.636 & 0.585 \\
& MCC
& 0.639 & 0.611 & 0.611 & 0.638 & 0.611 & 0.654 & 0.611 \\
& PR-AUC
& 0.668 & 0.665 & 0.665 & 0.678 & 0.665 & 0.658 & 0.665 \\
\hline

\multirow{4}{*}{Parkinson's Speech}
& \% Ret.
& 0.3 & 5.8 & 5.8 & 8.0 & 47.9 & 3.6 & 100.0 \\
& Min-F1
& 0.847 & 0.865 & 0.858 & 0.855 & 0.861 & 0.862 & 0.861 \\
& MCC
& 0.154 & 0.294 & 0.233 & 0.217 & 0.251 & 0.247 & 0.251 \\
& PR-AUC
& 0.885 & 0.936 & 0.936 & 0.925 & 0.937 & 0.938 & 0.937 \\
\hline

\multirow{4}{*}{Septic Shock (AdaBoost)}
& \% Ret.
& 6.7 & 3.5 & 3.5 & 13.5 & 43.4 & 2.8 & 100.0 \\
& Min-F1
& 0.240 & 0.190 & 0.190 & 0.216 & 0.300 & 0.069 & 0.190 \\
& MCC
& 0.065 & 0.043 & 0.043 & 0.109 & 0.187 & 0.010 & 0.043 \\
& PR-AUC
& 0.235 & 0.244 & 0.253 & 0.268 & 0.267 & 0.178 & 0.244 \\
\hline

\multirow{4}{*}{Septic Shock (GBoost)}
& \% Ret.
& 4.0 & 3.2 & 3.2 & 12.8 & 43.8 & 2.8 & 100.0 \\
& Min-F1
& 0.171 & 0.176 & 0.316 & 0.158 & 0.188 & 0.125 & 0.182 \\
& MCC
& 0.072 & 0.089 & 0.221 & 0.030 & 0.128 & 0.046 & 0.107 \\
& PR-AUC
& 0.260 & 0.296 & 0.341 & 0.264 & 0.306 & 0.162 & 0.256 \\
\hline

\end{tabular}%
}
\caption{Full comparison on the binary classification datasets. \% Ret.\ is the percentage of retained features; Min-F1 is the minority-class F1. NABFS (the XGBoost-based variant) is compared against six baselines.}
\label{tab:supp_binary}
\end{table*}

\begin{table*}[t]
\centering
\setlength{\tabcolsep}{4pt}
\begin{tabular}{@{}l p{0.72\linewidth}@{}}
\hline
\textbf{Component} & \textbf{Specification} \\
\hline
Platform
& Single local workstation (no cluster, no cloud) \\

CPU
& Intel Core i7-13700K (Raptor Lake, 13th generation), x86\_64 \\

Cores / threads
& 16 physical cores (8 performance + 8 efficiency) / 24 threads \\

Clock
& 800 MHz base -- 5.40 GHz max turbo (P-cores 5.3--5.4 GHz, E-cores 4.2 GHz) \\

Cache
& L2 24 MiB (10 instances), L3 30 MiB \\

Memory
& 64 GB RAM (62.6 GiB usable) + 8 GiB swap \\

GPU
& \textbf{None used.} Integrated Intel UHD Graphics 770 only; no discrete or CUDA-capable device present \\

Storage
& 1 TB NVMe SSD (Crucial P3, root filesystem); secondary 1 TB Samsung 980 PRO \\

Other accelerators
& None (no TPU, no FPGA) \\
\hline
\end{tabular}
\caption{Computational environment used for the experiments.}
\label{tab:hardware}
\end{table*}

\begin{table*}[t]
\centering
\small
\renewcommand{\arraystretch}{1.15}
\setlength{\tabcolsep}{5pt}

\begin{tabular}{p{0.28\textwidth} p{0.67\textwidth}}
\hline
\textbf{Models} & \textbf{Parameters} \\
\hline

XGBoost classifier &
objective="binary:logistic", eval\_metric="logloss", max\_depth=8,
learning\_rate=0.1, n\_estimators=200, base\_score=0.5,
n\_jobs=1, random\_state=replicate seed,
Default parameters from \texttt{xgboost.XGBClassifier} \\[4pt]

XGBoost regressor  &
objective="reg:squarederror", eval\_metric="rmse", max\_depth=8,
learning\_rate=0.1, n\_estimators=200, n\_jobs=1,
random\_state=replicate seed,
Default parameters from \texttt{xgboost.XGBRegressor} \\[4pt]

Random Forest (Boruta base estimator) &
n\_estimators(trees)=100, max\_depth=None,
class\_weight="balanced", n\_jobs=-1,
random\_state=replicate seed,
Default parameters from \texttt{sklearn.ensemble.RandomForestClassifier} \\[4pt]

Boruta &
n\_estimators="auto", verbose=0, random\_state=replicate seed,
Default parameters from \texttt{boruta.BorutaPy} \\[4pt]

Knockoff filter &
ksampler="gaussian", fstat="lasso", fdr=0.10,
Default parameters from \texttt{knockpy.knockoff\_filter.KnockoffFilter} \\[4pt]

\textbf{NABFS}  (Our Method) &
b(bootstrap replicates)=150, alpha=0.05,
num\_noise\_features(m)=1-5, sigma=1.0, standardize=True,
mode="bootstrap", importance\_type="gain", n\_jobs=-1
\\[4pt]

Wilcoxon signed-rank test &
alternative="greater", zero\_method="pratt",
Default parameters from \texttt{scipy.stats.wilcoxon} \\[4pt]

Bootstrap resampling &
replace=True, n\_samples=n, random\_state=replicate seed + i,
Default parameters from \texttt{sklearn.utils.resample} \\[4pt]

Feature standardization &
Default parameters from \texttt{sklearn.preprocessing.StandardScaler} \\[4pt]

Parallel execution &
n\_jobs=-1,
Default parameters from \texttt{joblib.Parallel} \\

\hline
\end{tabular}

\caption{Models and parameters. Only parameters set explicitly in the code are listed; everything else takes the library default from the class named in each cell.}
\label{tab:models_parameters}
\end{table*}















\section{Additional Diagnostic Experiments}\label{sec:additional_experiments}
To further probe the reliability of NABFS, we performed several sets of experiments on the Credit Card Fraud, Airline Customer Satisfaction, FN Cas9 (CRISPR), Food Delivery Times, and Student Performance datasets. The shadow-feature and real-versus-shuffled diagnostics were run on all of these datasets, the resampling-mode comparison on Credit Card Fraud, Airline Customer Satisfaction, and Food Delivery Times, and the null-calibration checks, which require label shuffling and repeated replicates, on Credit Card Fraud and Airline Customer Satisfaction. In these experiments, we analyzed four questions: does replacing the injected Gaussian noise with shadow (permuted) features affect the final selection; does the method respond only to true signal; does the per-replicate false-win rate match the theoretical value $1/(l+1)$; and is the family-wise error rate controlled under the null.

\subsection{Shadow Features versus Gaussian Noise}
The main method benchmarks each candidate feature against injected Gaussian noise. A natural alternative is to benchmark against shadow features, that is, randomly permuted copies of the real features, which preserve each feature’s marginal distribution while breaking its dependence on the response. Table~\ref{tab:shadow} compares the two approaches. The two variants select approximately the same number of features on each dataset, and give identical downstream predictive scores whenever the selected sets agree. On Credit Card Fraud, the shadow variant selects a slightly larger set (9 vs 6 features) and on Student Performance a slightly smaller set (6 vs 7), but in both cases with little difference in test scores. This suggests that NABFS is relatively insensitive to the choice of benchmark, and that the exchangeable shadow construction, which is the one our theory assumes, behaves in practice like the independent Gaussian construction used in the main experiments.
\begin{table*}[htbp]
\centering
\renewcommand{\arraystretch}{1.15}
\setlength{\tabcolsep}{5pt}
\begin{tabular}{llccl}
\hline
\textbf{Dataset} & \textbf{Benchmark} & \textbf{\# Feat.} & \textbf{\% Ret.} & \textbf{Key test metrics} \\
\hline
\multirow{2}{*}{Credit Card Fraud}
 & Gaussian noise & 6 & 20.7 & Min-F1 0.585, MCC 0.611, PR-AUC 0.665 \\
 & Shadow feature & 9 & 31.0 & Min-F1 0.585, MCC 0.611, PR-AUC 0.665 \\
\hline
\multirow{2}{*}{Airline Cust.\ Sat.}
 & Gaussian noise & 21 & 100.0 & Min-F1 0.834, MCC 0.624, PR-AUC 0.926 \\
 & Shadow feature & 21 & 100.0 & Min-F1 0.834, MCC 0.624, PR-AUC 0.926 \\
\hline
\multirow{2}{*}{FN Cas9 (CRISPR)}
 & Gaussian noise & 31 & 73.8 & Macro-F1 0.481, MCC 0.519, PR-AUC 0.940 \\
 & Shadow feature & 32 & 76.2 & Macro-F1 0.481, MCC 0.519, PR-AUC 0.940 \\
\hline
\multirow{2}{*}{Food Delivery Times}
 & Gaussian noise & 2 & 14.3 & RMSE 12.404, MAE 9.799, $R^2$ 0.636 \\
 & Shadow feature & 2 & 14.3 & RMSE 12.404, MAE 9.799, $R^2$ 0.636 \\
\hline
\multirow{2}{*}{Student Performance}
 & Gaussian noise & 7 & 17.9 & RMSE 3.507, MAE 2.614, $R^2$ 0.171 \\
 & Shadow feature & 6 & 15.4 & RMSE 3.550, MAE 2.664, $R^2$ 0.150 \\
\hline
\end{tabular}
\caption{Shadow-feature versus Gaussian-noise benchmark. \# Feat.\ is the number of selected features and \% Ret.\ the corresponding percentage retained. When the two benchmarks select the same features, the downstream metrics are identical; the small differences on Credit Card Fraud and Student Performance leave the test scores essentially unchanged.}
\label{tab:shadow}
\end{table*}

\subsection{Selection on Real versus Shuffled Labels}
A basic requirement for any selection procedure to be reasonable is that it recovers features only in the presence of true signal. We can evaluate this criterion for NABFS by comparing the number of features selected on the true data with those selected after randomly permuting the response, which eliminates any possibility of true association. In Table~\ref{tab:realvsshuffled}, we report the number of features selected by NABFS for each of our test datasets on the true and shuffled labels.
We observe that with shuffled labels the number of features selected drops precipitously to zero on every dataset, demonstrating that NABFS is appropriately conservative in the absence of true signal. On the real labels, by contrast, NABFS consistently selected a reasonable number of features from every dataset.
,
\begin{table}[htbp]
\centering
\small
\renewcommand{\arraystretch}{1.15}
\setlength{\tabcolsep}{8pt}
\begin{tabular}{lcc}
\hline
\textbf{Dataset} & \textbf{Real labels} & \textbf{Shuffled labels} \\
\hline
Credit Card Fraud    & 5  & 0 \\
Airline Cust.\ Sat.   & 21 & 0 \\
FN Cas9 (CRISPR)     & 30 & 0 \\
Food Delivery Times  & 2  & 0 \\
Student Performance  & 7  & 0 \\
\hline
\end{tabular}
\vspace{0.1cm}
\caption{Number of features selected by NABFS on the true data and after shuffling the response. Selection collapses to zero once the signal is destroyed, confirming that the method does not fire on noise.}
\label{tab:realvsshuffled}
\end{table}

\subsection{Bootstrap versus Subsample Resampling}
Each round of NABFS uses a new resample of the training data before fitting the model. We tested two modes of resampling: the bootstrap (sampling with replacement) used in the main experiments, and subsampling (sampling without replacement). Table~\ref{tab:resample} shows the Wilcoxon signed-rank false-positive rate as a function of the resampling mode used. The false-positive rate should be close to or below the nominal $0.05$ level. Both modes achieve a $0.000$ false-positive rate on all three datasets, and the two modes behave identically. The two modes are therefore largely interchangeable, with the bootstrap mode being slightly preferable as a default.

\begin{table}[!htbp]
\centering
\small
\renewcommand{\arraystretch}{1.15}
\setlength{\tabcolsep}{8pt}
\begin{tabular}{lcc}
\hline
\textbf{Dataset} & \textbf{Bootstrap} & \textbf{Subsample} \\
\hline
Credit Card Fraud    & 0.000 & 0.000 \\
Airline Cust.\ Sat.   & 0.000 & 0.000 \\
Food Delivery Times  & 0.000 & 0.000 \\
\hline
\end{tabular}
\vspace{0.1cm}
\caption{Wilcoxon signed-rank false-positive rate under the two resampling modes (target $\le 0.05$). The two modes behave identically.}
\label{tab:resample}
\end{table}

\subsection{Empirical Check of the Exact Null Rate \texorpdfstring{$1/(l+1)$}{1/(l+1)}}
Our theory predicts that, under the null, a candidate feature will beat the maximum of $l$ noise features with probability exactly $1/(l+1)$ per replicate. We can check this directly on Credit Card Fraud and Airline Customer Satisfaction by measuring the observed per-replicate win rate for several values of $l$. Table~\ref{tab:nullrate} reports two independent runs per dataset along with the theoretical value. The observed win rate tracks closely with $1/(l+1)$ across all values of $l$, as theory predicts. On Airline Customer Satisfaction, it stays at or below the theoretical value, while on Credit Card Fraud it is very close to it, slightly below at $l=1,2$ and slightly above at $l=3,5$. The agreement is tightest for $l=1$, where the observed rate is near the theoretical $0.5$ on Credit Card Fraud, and these small deviations are consistent with finite-sample variation and the mild dependence induced by bootstrap resampling.

\begin{table}[!htbp]
\centering
\small
\renewcommand{\arraystretch}{1.15}
\setlength{\tabcolsep}{6pt}
\begin{tabular}{lccccc}
\hline
\textbf{Dataset} & \textbf{$l$} & \textbf{Run 1} & \textbf{Run 2} & \textbf{$1/(l+1)$} \\
\hline
\multirow{4}{*}{Credit Card Fraud}
 & 1 & 0.506 & 0.484 & 0.500 \\
 & 2 & 0.259 & 0.273 & 0.333 \\
 & 3 & 0.262 & 0.282 & 0.250 \\
 & 5 & 0.184 & 0.195 & 0.167 \\
\hline
\multirow{4}{*}{Airline Cust.\ Sat.}
 & 1 & 0.279 & 0.396 & 0.500 \\
 & 2 & 0.178 & 0.269 & 0.333 \\
 & 3 & 0.158 & 0.232 & 0.250 \\
 & 5 & 0.118 & 0.179 & 0.167 \\
\hline
\end{tabular}
\vspace{0.1cm}
\caption{Observed per-replicate win rate against the theoretical value $1/(l+1)$, for two independent runs. }
\label{tab:nullrate}
\end{table}

\subsection{Family-Wise Error Rate under the Null}
Finally, we evaluated the family-wise error rate (FWER) under shuffled labels on Credit Card Fraud and Airline Customer Satisfaction using three testing rules: the Wilcoxon signed-rank test used in NABFS, a naive sign test, and an empirical-null test. Table~\ref{tab:fwer} reports the results, which should be no higher than $0.05$. On both data sets, the Wilcoxon signed-rank test and the empirical-null test have an FWER of $0.000$, while the naive sign test has an uncontrolled FWER of $0.850$ and $1.000$, respectively. This is consistent with the analysis in Section~\ref{potr}, where the signed-rank and empirical-null procedures are shown to control the error rate while the naive sign test is not, which is why NABFS uses the former.
\begin{table}[!htbp]
\centering
\small
\renewcommand{\arraystretch}{1.15}
\setlength{\tabcolsep}{6pt}
\begin{tabular}{lccc}
\hline
\textbf{Dataset} & \textbf{Wilcoxon SR} & \textbf{Sign test} & \textbf{Empirical null} \\
\hline
Credit Card Fraud   & 0.000 & 0.850 & 0.000 \\
Airline Cust.\ Sat.  & 0.000 & 1.000 & 0.000 \\
\hline
\end{tabular}
\vspace{0.2 cm}
\caption{Empirical family-wise error rate under shuffled labels (target $\le 0.05$). The Wilcoxon signed-rank and empirical-null tests control the error rate, while the naive sign test does not.}
\label{tab:fwer}
\end{table}

\end{document}